\documentclass{article}

     \usepackage[final]{neurips_2020}

\usepackage[utf8]{inputenc} % allow utf-8 input
\usepackage[T1]{fontenc}    % use 8-bit T1 fonts
\usepackage{hyperref}       % hyperlinks
\usepackage{url}            % simple URL typesetting
\usepackage{booktabs}       % professional-quality tables
\usepackage{amsfonts}
\usepackage{subcaption}% blackboard math symbols
\usepackage{nicefrac}       % compact symbols for 1/2, etc.
\usepackage{microtype}      % microtypography

\usepackage{algorithm}% http://ctan.org/pkg/algorithms
\usepackage{algpseudocode}% http://ctan.org/pkg/algorithmicx

\usepackage{mathtools}
\usepackage{graphicx}
\usepackage{amssymb}
\usepackage{bbm}
\usepackage{amsthm}
\usepackage{xpatch}
\usepackage{mathrsfs}

\usepackage{url}
\usepackage{array}
\usepackage{wrapfig}
\usepackage{multirow}
\usepackage{tabularx}
\usepackage[normalem]{ulem} % to use \sout
\usepackage{enumerate}
\usepackage{enumitem}

\usepackage[usenames,dvipsnames]{xcolor}
\newcommand\cut[1]{{}}  % cut out
\makeatletter
\xpatchcmd{\algorithmic}{\itemsep\z@}{\itemsep=.5ex plus.5pt}{}{}
\makeatother
\usepackage{mathtools}
\mathtoolsset{showonlyrefs}

\newcommand{\sA}{{\mathscr{A}}}

\newcommand{\cA}{{\mathcal{A}}}

\newcommand{\cD}{{\mathcal{D}}}
\newcommand{\cE}{{\mathcal{E}}}
\newcommand{\cF}{{\mathcal{F}}}
\newcommand{\cG}{{\mathcal{G}}}

\newcommand{\cL}{{\mathcal{L}}}
\newcommand{\cM}{{\mathcal{M}}}
\newcommand{\cN}{{\mathcal{N}}}
\newcommand{\cO}{{\mathcal{O}}}
\newcommand{\cP}{{\mathcal{P}}}

\newcommand{\cS}{{\mathcal{S}}}

\newcommand{\cU}{{\mathcal{U}}}

\newcommand{\cX}{{\mathcal{X}}}
\newcommand{\cY}{{\mathcal{Y}}}
\newcommand{\cZ}{{\mathcal{Z}}}

\newcommand{\RR}{\mathbb{R}}

\newcommand{\unif}{{\mathrm{Unif}}}
\newcommand{\poly}{{\mathrm{poly}}} % polynomial
\renewcommand{\iota}{I}

\DeclareMathOperator*{\argmax}{argmax}

\newcommand{\bc}{\begin{center}}
\newcommand{\ec}{\end{center}}

\newcommand{\bdm}{\begin{displaymath}}
\newcommand{\edm}{\end{displaymath}}

\newcommand{\beq}{\begin{equation}}
\newcommand{\eeq}{\end{equation}}

\newcommand{\bfl}{\begin{flushleft}}
\newcommand{\efl}{\end{flushleft}}

\newcommand{\bt}{\begin{tabbing}}
\newcommand{\et}{\end{tabbing}}

\newcommand{\beqn}{\begin{align}}
\newcommand{\eeqn}{\end{align}}

\newcommand{\beqs}{\begin{align*}} % no equation numbers
\newcommand{\eeqs}{\end{align*}}  % no equation numbers

\newtheorem{theorem}{Theorem}

\newtheorem{definition}{Definition}
\newtheorem{corollary}{Corollary}

\newtheorem{lemma}{Lemma}
\newtheorem{proposition}{Proposition}

\makeatletter
\renewcommand*\env@matrix[1][\arraystretch]{%
  \edef\arraystretch{#1}%
  \hskip -\arraycolsep
  \let\@ifnextchar\new@ifnextchar
  \array{*\c@MaxMatrixCols c}}
\makeatother

\newcommand{\ulo}{\cU\cL\cO}
\newcommand{\tsr}{\texttt{TSR}}

\newenvironment{itemize*}%
{\begin{itemize}[leftmargin=*,topsep=5pt]%
		\setlength{\itemsep}{1pt}%
		\setlength{\parskip}{1pt}}%
	{\end{itemize}}

\title{Provably Efficient Exploration for Reinforcement Learning Using Unsupervised Learning
	\thanks{
		Correspondence to: Simon S. Du <ssdu@cs.washington.edu>, Lin F. Yang <linyang@ee.ucla.edu>
	}
	}

\author{%
Fei Feng \\
University of California, Los Angeles\\
\texttt{fei.feng@math.ucla.edu} \\
\And
Ruosong Wang \\
Carnegie Mellon University\\
\texttt{ruosongw@andrew.cmu.edu} \\
\And
Wotao Yin\\
University of California, Los Angeles\\
\texttt{wotaoyin@math.ucla.edu}\\
\And
Simon S. Du\\
University of Washington\\
\texttt{ssdu@cs.washington.edu}
\And
Lin F. Yang\\
University of California, Los Angeles\\
\texttt{linyang@ee.ucla.edu}
}

\begin{document}

\maketitle

\begin{abstract}
Motivated by the prevailing paradigm of using unsupervised learning for efficient exploration in reinforcement learning (RL) problems~\citep{tang2017exploration,bellemare2016unifying}, we investigate when this paradigm is provably efficient.
%We study how to use unsupervised learning for efficient exploration in reinforcement learning with rich observations generated from a small number of latent states. 
We study episodic Markov decision processes with rich observations generated from a small number of latent states.
We present a general algorithmic framework that is built upon two components: an unsupervised learning algorithm and a no-regret tabular RL algorithm. 
Theoretically, we prove that as long as the unsupervised learning algorithm enjoys a polynomial sample complexity guarantee, we can find a near-optimal policy with sample complexity polynomial in the number of latent states, which is significantly smaller than the number of observations. Empirically, we instantiate our framework on a class of hard exploration problems to demonstrate the practicality of our theory.
%We complement our theory with an empirical evaluation on a class of hard exploration problems. 
%Our result gives theoretical justification to the prevailing paradigm of using unsupervised learning for efficient exploration~\citep{tang2017exploration,bellemare2016unifying}.
\end{abstract}

\section{Introduction}\label{sec:intro}
%Challenges in RL: huge state space and efficient exploration

Reinforcement learning (RL) is the framework of learning to control an unknown system through trial and error. It takes as inputs the observations of the environment and outputs a policy, i.e., a mapping from observations to actions, to maximize the cumulative rewards. To learn a near-optimal policy, it is critical to sufficiently explore the environment and identify all the opportunities for high rewards. However, modern RL applications often need to deal with huge observation spaces such as those consist of images or texts, which makes it challenging or impossible (if there are infinitely many observations) to fully explore the environment in a direct way. In some work, function approximation scheme is adopted such that essential quantities for policy improvement, e.g. state-action values, can be generalized from limited observed data to the whole observation space. However, the use of function approximation alone does not resolve the exploration problem~\citep{Du2020Is}.
%To obtain a near-optimal policy, the agent needs to sufficiently explore the environment and identify all the opportunities for high rewards.
%How to efficiently explore for RL with a huge observation space is a central challenge.

%Empirically
To tackle this issue, multiple empirically successful strategies are developed~\citep{tang2017exploration,bellemare2016unifying,pathak2017curiosity,azizzadenesheli2018efficient,lipton2018bbq,fortunato2018noisy,osband2016generalization}.
Particularly, in \cite{tang2017exploration} and \cite{bellemare2016unifying}, %count-based exploration is proposed, a technique inspired by the upper-confidence-bound (UCB) from bandit and the tabular implementation in RL. At a high level, 
the authors use state abstraction technique to reduce the problem size. They construct a mapping from observations to a small number of hidden states and devise exploration on top of the latent state space rather than the original observation space.
%these two methods apply state abstraction to reduce the problem size, i.e.,  then add exploration bonus on top of the hidden states instead of obser. 

To construct such a state abstraction mapping, practitioners often use \emph{unsupervised learning}. 
The procedure has the following steps: collect a batch of observation data, apply unsupervised learning to build a mapping, use the mapping to guide exploration and collect more data, and repeat.
Empirical study evidences the effectiveness of such an approach at addressing hard exploration problems (e.g., the infamous Montezuma's Revenge).
However, it has not been theoretically justified. 
In this paper, we aim to answer this question:
\begin{center}
	\emph{Is exploration driven by unsupervised learning in general \textbf{provably efficient}}?
\end{center} 
The generality includes the choice of unsupervised learning algorithms, reinforcement learning algorithms, and the condition of the problem structure. 

%tabular
We first review some existing theoretical results on provably efficient exploration. 
More discussion about related work is deferred to appendix. For an RL problem with finitely many states,
%(but unknown transition and rewards),
there are many algorithms with a tabular implementation that learn to control efficiently.
%e.g., \citet{kearns2002near,brafman2002r,jin2018q, jaksch2010near, azar2017minimax}, and the needed number of samples depends polynomially on the size of the state space. 
%For instance, \citet{kearns2002near,brafman2002r,strehl2006pac, jaksch2010near,dann2015sample, agrawal2017posterior, jin2018q,  kakade2018variance,zanette2019tighter} 
These algorithms can  learn a near-optimal policy using a number of samples polynomially depending on the size of the state space.
However, if we 
%use a tabular implementation and 
directly apply these algorithms to rich observations cases by treating each observation as a state,
%algorithms in \citep{kearns2002near,brafman2002r, jin2019provably, azar2017minimax} 
the sample complexities are polynomial in the cardinality of the observation space. Such a dependency is unavoidable without additional structural assumptions~\citep{jaksch2010near}.
%If we directly work on the observation space without any abstraction,  achieve exploration with sample complexities polynomial in the dimension of the observation space and therefore, are not applicable to large-scale problems. polynomial sample a tabular setting, the theory of provably efficient exploration has been well established \citep{kearns2002near,brafman2002r}.
%However, since no state abstraction is applied, the sample complexity of these algorithms are polynomial in the dimension of the observation space, which can not be ameliorated without further assumptions~\citep{jaksch2010near}. Therefore, they are not applicable in our case.
%are difficult to adapt to RL problems with a huge observation space since the sample complexity of these algorithms are \emph{polynomial} in the dimension of the observation space.
%Furthermore, unfortunately, without additional structural assumptions, such a dependency is unavoidable~\citep{jaksch2010near}.
%huge state space theory
If structural conditions are considered, for example, observations are generated from a small number of latent states \citep{krishnamurthy2016pac,jiang2017contextual,dann2018oracle,du2019provably}, then the sample complexity only scales polynomially with the number of hidden states. 
Unfortunately, the correctness of these algorithms often requires strict assumptions (e.g., deterministic transitions, reachability) that may not be satisfied in many real applications. 
\paragraph{Our Contributions}
In this paper we study RL problems with rich observations generated from a small number of latent states for which an unsupervised learning subroutine is used to guide exploration.
We summarize our contributions below.
\begin{itemize*}
\item We propose a new algorithmic framework for the Block Markov Decision Process (BMDP) model~\citep{du2019provably}.
We combine an unsupervised learning oracle and a tabular RL algorithm in an organic way to find a near-optimal policy for a BMDP.
The unsupervised learning oracle is an abstraction of methods used in \cite{tang2017exploration,bellemare2016unifying} and widely used statistical generative models. 
%We provide concrete examples of this oracle and analyze its statistical properties.
Notably, our framework can take almost \emph{any} unsupervised learning algorithms and tabular RL algorithms as subroutines. 
%and thus provides a generic way of constructing efficient algorithms for RL problems with a huge state space from tabular RL algorithms.
%\item We define a general notion, unsupervised learning oracle, which is an abstraction of methods used in \cite{tang2017exploration,bellemare2016unifying} and standard statistical generative models. We provide concrete examples of this oracle and analyze its statistical properties.
\item Theoretically, we prove that as long as the unsupervised learning oracle and the tabular RL algorithm each has a polynomial sample complexity guarantee, our framework finds a near-optimal policy with sample complexity polynomial in the number of latent states, which is significantly smaller than the number of possible observations (cf. Theorem~\ref{thm:main}). 
To our knowledge, this is the \emph{first} provably efficient method for RL problems with huge observation spaces that uses unsupervised learning for exploration.
Furthermore, our result does not require additional assumptions on transition dynamics as used in \cite{du2019provably}.
Our result theoretically sheds light on the success of the empirical paradigms used in \cite{tang2017exploration,bellemare2016unifying}. 

\item We instantiate our framework with particular unsupervised learning algorithms and tabular  RL algorithms on hard exploration environments with rich observations studied in \cite{du2019provably}, and compare with other methods tested in \cite{du2019provably}.
Our experiments demonstrate our method can significantly outperform existing methods on these environments.

%is a general reduction scheme that turns a polynomial sample complexity unsupervised learning oracle and a no-regret tabular RL algorithm into a provably efficient algorithm that 
%finds a near-optimal policy with a polynomial number of samples  . 
%for problems with huge observation spaces but much smaller latent state spaces .
%
%we demonstrate that if the unsupervised learning oracle has polynomial sample complexity and the tabular RL algorithm satisfy the no-regret property, our framework provably finds a near-optimal policy with polynomial number of samples.

\end{itemize*}

\paragraph{Main Challenge and Our Technique}
We assume there is an unsupervised learning oracle (see formal definition in Section \ref{sec:framework}) which can be applied to learn decoding functions and the accuracy of learning increases as more training data are fed. 
The unsupervised learning algorithm can only guarantee good performance with respect to the input distribution that generates the training data.
Unlike standard unsupervised learning where the input distribution is fixed, in our problem, the input distribution depends on our policy.
On the other hand, the quality of a policy depends on whether the unsupervised learning oracle has (approximately) decoded the latent states.
This interdependency is the main challenge we need to tackle in our algorithm design and analysis.

Here we briefly explain our framework. 
%We assume there is an unsupervised learning oracle (see formal definition in Section \ref{sec:framework}) which can be applied to learn decoding functions and the accuracy of learning increases as more training data are fed. 
%For an RL problem with fully observable states (but unknown transition and rewards), there are many efficient no-regret online algorithms that learn to control the system efficiently.
%For instance, \citet{jin2018q, jaksch2010near, azar2017minimax} learn an approximate policy using a number of samples polynomially depending on the size of the state space.
%If we apply these algorithms directly to our setting, i.e. treat each observation as a state, then the sample complexity is tremendous when the observation space is huge.
%needed number of samples is proportional to the number of distinct observations.
%Unfortunately, in most cases, the number of observations can be much larger than the number of true states or even infinite.
%In order to effectively adopt the well-balanced exploitation-exploration structure in these no-regret algorithms, 
%we propose an algorithm that reduces the problem size by learning the correct decoding functions while collecting good rewards at the same time.
%The high-level idea is as follows: 
Let $\cM$ be the MDP with rich observations. We form an auxiliary MDP $\cM'$ whose state space is the latent state space of $\cM$. Our idea is to simulate the process of running a no-regret tabular RL algorithm $\sA$ directly on $\cM'$. For each episode, $\sA$ proposes a policy $\pi$ for $\cM'$ and expects a trajectory of running $\pi$ on $\cM'$ for updating and then proceeds. To obtain such a trajectory, we design a policy $\phi$ for $\cM$ as a composite of $\pi$ and some initial decoding functions. We run $\phi$ on $\cM$ to collect observation trajectories. 
%With $\pi$ and the help of some decoding functions which maps each observation to a state, we run on $\cM$ and collect observation trajectories. 
Although the decoding functions may be inaccurate initially, they can still help us collect observation samples for later refinement.
%we run a no-regret algorithm on an auxiliary MDP $\cM'$ whose state space corresponds to the true state space of the rich observation MDP $\cM$. Each time the no-regret algorithm proposes a policy $\pi$ on $\cM'$. We run $\pi$ on the rich observation MDP $\cM$ with the help of arbitrary decoding functions, i.e., that maps observations to a state in $\cM'$. The idea is that even if the decoding function is not correct, it can still help us collect observations from $\cM$.
After collecting sufficient observations, we apply the unsupervised learning oracle to retrain decoding functions and then update $\phi$ as a composite of $\pi$ and the newly-learned functions and repeat running $\phi$ on $\cM$.
%with a new composite policy of $\pi$ with the updated functions.
%Since there are only a few true states, 
After a number of iterations (proportional to the size of the latent state space), with the accumulation of training data, decoding functions are trained to be fairly accurate on recovering latent states, especially those $\pi$ has large probabilities to visit. This implies that running the latest $\phi$ on $\cM$ is almost equivalent to running $\pi$ on $\cM'$. Therefore, we can obtain a state-action trajectory with high accuracy as the algorithm $\sA$ requires.
%The above process of running $\pi$ and learning decoding functions converges 
%after every latent state (along the path of $\pi$) gets observed enough many times. In total, it requires 
%after a number of iterations proportional to the size of the latent state space. 
%After convergence, we collect a state-action trajectory using $\pi$ and the latest learned decoding functions and feed the trajectory to $\sA$ for a new policy. It can be proved that the trajectory with high probability is as generated by running $\pi$ directly on the auxiliary MDP $\cM'$. The whole process simulates the running of $\sA$ on the auxiliary MDP, $\cM'$.
Since $\sA$ is guaranteed to output a near-optimal policy after a polynomial (in the size of the true state-space) number of episodes, our algorithm uses polynomial number of samples as well.

\section{Related Work}\label{sec:rel}
In this section, we review related provably efficient RL algorithms.
We remark that we focus on environments that require explicit exploration.
With certain assumptions of the environment, e.g., the existence of a good exploration policy or the distribution over the initial state is sufficiently diverse, one does not need to explicitly explore~\citep{munos2005error,antos2008learning,geist2019theory,kakade2002approximately,bagnell2004policy,scherrer2014local,agarwal2019optimality,yang2019theoretical,chen2019information}.
%These works assume the existence of a good exploration policy or the distribution over the initial state is sufficiently diverse.
Without these assumptions, the problem can require an exponential number of samples, especially for policy-based methods~\citep{Du2020Is}.

%Tabular
Exploration is needed even in the most basic tabular setting.
There is a substantial body of work on provably efficient tabular RL~\citep{agrawal2017optimistic,jaksch2010near,kakade2018variance,azar2017minimax,kearns2002near,dann2017unifying,strehl2006pac,jin2018q,simchowitz2019non,zanette2019tighter}.
A common strategy is to use UCB bonus to encourage exploration in less-visited states and actions.
%metric 
One can also study RL in metric spaces~\citep{pazis2013pac,song2019efficient, 8919864}.
However, in general, this type of algorithms has an exponential dependence on the state dimension.

%Wen Zheng's paper 
To deal with huge observation spaces, one might use function approximation.
\citet{wen2013efficient} proposed an algorithm, optimistic constraint propagation (OCP), which enjoys polynomial sample complexity bounds for a family of $Q$-function classes, including the linear function class as a special case.
But their algorithm can only handle deterministic systems, i.e., both transition dynamics and rewards are deterministic.
The setting is recently generalized by \citet{du2019provablyQ} to environments with low variance and by \cite{du2020agnostic} to the agnostic setting.
%However, it is still unclear whether this type of algorithms can deal with general stochastic environments.
\cite{li2011knows} proposed a Q-learning algorithm which requires the Know-What-It-Knows oracle. But it is in general unknown how to implement such an oracle.

%CDP related papers
Our work is closely related to a sequence of works which assumes the transition has certain low-rank structure~\citep{krishnamurthy2016pac,jiang2017contextual,dann2018oracle,sun2019model,du2019provably,jin2019provably,yang2019sample}.
%Du et al.
The most related paper is \cite{du2019provably} which also builds a state abstraction map.
Their sample complexity depends on two quantities of the transition probability of the hidden states: \emph{identifiability} and \emph{reachability}, which may not be satisfied in many scenarios.
Identifiability assumption requires that the $L_1$ distance between the posterior distributions (of previous level's hidden state, action pair) given any two different hidden states is strictly larger than some constant (Assumption 3.2 in \cite{du2019provably}). 
This is an inherent necessary assumption for the method in \cite{du2019provably} as they need to use the posterior distribution to distinguish hidden states.
Reachability assumption requires that there exists a constant such that for every hidden state, there exists a policy that reaches the hidden state with probability larger than this constant (Definition 2.1 in \cite{du2019provably}).
Conceptually, this assumption is not needed for finding a near-optimal policy because if one hidden state has negligible reaching probability, one can just ignore it.
Nevertheless, in \citet{du2019provably}, the reachability assumption is also tied with building the abstraction map.
Therefore, it may not be removable if one uses the strategy in \cite{du2019provably}.
In this paper, we show that given an unsupervised learning oracle, one does not need the identifiability and reachability assumptions for efficient exploration.

\section{Preliminaries}\label{sec:setting}
\paragraph{Notations}
%We use small letters (e.g. $s,a,h, p(s'|s,a)$) for scalars, small boldface letters (e.g. $\vp(\cdot|s,a), \vq(\cdot|s)$) for vectors, capital letters (e.g., $P_h, P^{\pi}$) for matrices or (e.g. $Q^{\pi}, V^{\pi}, R, G$) for functions or (e.g. $K,L,M$) some specific scalar parameters, which will be clear in the context, and calligraphic letters (e.g., $\cS,\cA,\cX, \cP$) for sets and collections. 
Given a set $\cA$, we denote by $|\cA|$ the cardinality of $\cA$, $\cP(\cA)$ the set of all probability distributions over $\cA$, and $\unif(\cA)$ the uniform distribution over $\cA$. We use $[h]$ for the set $\{1,2,\dots,h\}$ and $f_{[h]}$ for the set of functions $\{f_1,f_2,\dots,f_h\}$. Given two functions $f:\cX\rightarrow\cY$ and $g:\cY\rightarrow\cZ$, their composite is denoted as $g\circ f:\cX\rightarrow\cZ$.
%The simplex set in $\RR^k$ is denoted by $\delta_f^k$.
%:=\{(x_1,x_2,\dots,x_k)^T~|~\sum_{i=1}^k x_i=1, x_i\geq 0\}$. 
%Given an event $\cE$, we denote by $\Pr(\cE)$ the probability that $\cE$ occurs. 
%Given a random variable $X$, we denote by $X\sim q(\cdot)$ if $X$ follows a distribution with density $q(\cdot)$. 
%We use $\cO$ and $\Theta$ to represent the leading order in upper and minimax bound, respectively. We use $\poly(\cdot)$ for the polynomial dependency.
%We use $\cO, \Omega$ and $\Theta$ to denote leading orders in upper, lower, and minimax lower bounds, respectively; and we use $\widetilde{\cO}, \widetilde{\Omega}$ and $\widetilde{\Theta}$ to hide the polylog factors.

\paragraph {Block Markov Decision Process} 
We consider a Block Markov Decision Process (BMDP), which is first formally introduced in \citet{du2019provably}. A BMDP is described by a tuple $\cM:=(\cS,\cA,\cX, \cP, r, f_{[H+1]}, H)$. $\cS$ is a finite unobservable \emph{latent state space}, $\cA$ is a finite action space, and $\cX$ is a possibly infinite observable context space. $\cX$ can be partitioned into $|\cS|$ disjoint blocks $\{\cX_s\}_{s\in\cS}$, where each block $\cX_s$ corresponds to a unique state $s$. $\cP$ is the collection of the \emph{state-transition probability} $p_{[H]}(s'|s,a)$ and the \emph{context-emission distribution} $q(x|s)$ for all $s,s'\in\cS, a\in\cA, x\in\cX$. $r:[H]\times\cS\times\cA\rightarrow[0,1]$ is the reward function. $f_{[H+1]}$ is the set of decoding functions, where $f_h$ maps every observation at level $h$ to its true latent state. Finally, $H$ is the length of horizon.
When $\cX=\cS$, this is the usual MDP setting.

%\simon{Why we need this?}
%\begin{remark}
%BMDP is a special type of Partially Observed Markov Decision Process (POMDP) (\cite{spaan2012partially}). In general, an observation of POMDP can infer various states. In BMDP, one observation only corresponds to a unique state. As a consequence, the transitions among observations are Markovian in BMDP: 
%\begin{align}
%&\Pr(x_{h+1}|a_h, x_h,\cdots,a_1, x_1)\\
%= &q\big(x_{h+1}|s_{h+1})\big)\Pr\big(s_{h+1}|a_h, x_h,\cdots,x_1\big)\\
%= &q\big(x_{h+1}|s_{h+1})\big)\Pr\big(s_{h+1}|a_h, s_h,\cdots,s_1\big)\\
%= &q\big(x_{h+1}|s_{h+1})\big)p_h\big(s_{h+1}|a_h, s_h\big)=\Pr(x_{h+1}|x_h,a_h).
%\end{align}
%In real applications, $\cX$ can be a high-dimensional rich observation space and $\cS$ is a dimension-reduced compact representation space, where the latter congregates analogues in former and abstracts representatives.
%\end{remark}

For each episode, the agent starts at level 1 with the initial state $s_1$ and takes $H$ steps to the final level $H+1$. 
We denote by $\cS_h$ and $\cX_h$ the set of possible states and observations at level $h\in[H+1]$, respectively. 
%We suppose that $\cS_h\cap\cS_{h'}=\emptyset$ for all $h\neq h'$ and $h,h'\in[H+1]$. This can be achieved by adding step label into the states. By such a design, we have every state $s\in\cS$ only appears in one level. 
%and $P_h$ the transition probability from level $h$ to $h+1$. 
At each level $h\in[H+1]$, the agent has no access to the true latent state $s_h\in\cS_h$ but an observation $x_h\sim q(\cdot|s_h)$. An action $a_h$ is then selected following some policy $\phi:[H]\times\cX\rightarrow \cP(\cA)$. As a result, the environment evolves into a new state $s_{h+1}\sim p_h(\cdot|s_h,a_h)$ and the agent receives an instant reward $ r(h, s_h,a_h)$. A trajectory has such a form: $\{s_1, x_1, a_1, \dots, s_H, x_H, a_H, s_{H+1}, x_{H+1}\}$, where all state components are unknown.

\paragraph{Policy} Given a BMDP $\cM:=(\cS,\cA,\cX,\cP,r,f_{[H+1]},H)$, there is a corresponding MDP $\cM':=(\cS,\cA,\cP,r,H)$, which we refer to as \emph{the underlying MDP} in the later context. A policy on $\cM$ has a form $\phi:[H]\times\cX\rightarrow \cP(\cA)$ and a policy on $\cM'$ has a form $\pi:[H]\times\cS\rightarrow\cP(\cA)$. Given a policy $\pi$ on $\cM'$ and a set of functions $\hat{f}_{[H+1]}$ where $\hat{f}_h:\cX_h\rightarrow\cS_h, \forall~ h\in[H+1]$, we can induce a policy on $\cM$ as $\pi\circ\hat{f}_{[H+1]}=:\phi$ such that 
$
\phi(h, x_h)=\pi(h, \hat{f}_h(x_h)), ~\forall~ x_h\in\cX_h, ~h\in[H].
$
If $\hat{f}_{[H+1]}=f_{[H+1]}$, then $\pi$ and $\phi$ are equivalent in the sense that they induce the same probability measure over the state-action trajectory space.

Given an MDP, the value of a policy $\pi$ (starting from $s_1$) is defined as $ V_1^{\pi}%(s_1)
    =~\mathbb{E}^{\pi} \left[\sum_{h=1}^{H}r(h,s_h,a_h)\Big|s_1\right]$,
\cut{\begin{align}
    V_1^{\pi}%(s_1)
    =~\mathbb{E}^{\pi} \left[\sum_{h=1}^{H}r(h,s_h,a_h)\Big|s_1\right]
    =~\sum_{h=1}^H\sum_{s\in\cS_h}\sum_{a\in\cA} P^{\pi}_h(s|s_1)\pi(a|h,s)r(h,s,a),
\end{align}}
%where the expectation is taken over the whole trajectory following $\pi$, $P^{\pi}_h(s|s_1)$ denotes the probability of reaching state $s$ at level $h$ starting from $s_1$ following $\pi$, and $\pi(a|h,s)$ is the probability of selecting action $a$ at state $s$ and level $h$. 
A policy that has the maximal value is an \emph{optimal policy} and the \emph{optimal value} is denoted by $V_1^*$, i.e., $V_1^*=\max_{\pi} V_1^{\pi}$. Given $\varepsilon>0$, we say $\pi$ is \emph{$\varepsilon$-optimal} if $V_1^*-V_1^\pi\leq\varepsilon$. Similarly, given a BMDP, we define the value of a policy $\phi$ (starting from $s_1$) as: $ V_1^{\phi}%(s_1)
    =~\mathbb{E}^{\phi} \left[\sum_{h=1}^{H}r(h,s_h,a_h)\Big|s_1\right]$,
\cut{\begin{align}
    V_1^{\phi}%(s_1)
    =~\mathbb{E}^{\phi} \left[\sum_{h=1}^{H}r(h,s_h,a_h)\Big|s_1\right]
    =~\sum_{h=1}^H\sum_{s\in\cS_h}\sum_{x\in\cX_s}\sum_{a\in\cA} P^{\phi}_h(s|s_1)q(x|s)\phi(a|h, x)r(h,s,a).
\end{align}}
%where the expectation is taken following $\phi$, $P^{\phi}_h(s|s_1)$ is the reaching probability to $s$ at level $h$ following $\phi$, and $\phi(a|h,x)$ denotes the probability of selecting action $a$ at level $h$ with observation $x$. 
The notion of optimallity and $\varepsilon$-optimality are similar to MDP. 
%In particular, if $\pi^*$ is optimal for the underlying MDP, then $\pi^*\circ f_{[H+1]}$ is optimal for the BMDP, since the reward function $r$ depends on states rather than observations.

%\textcolor{blue}{Furthermore, given a set of policies $\Pi$, we denote by $\unif(\Pi)$ a policy that for each episode, the agent first uniformly randomly selects a policy from $\Pi$ then follows that policy for $H$ steps. Note that, at different episodes, the selected policies from $\Pi$} 

\section{A Unified Framework for Unsupervised Reinforcement Learning}\label{sec:framework}
\begin{algorithm}[t]
	\caption{A Unified Framework for Unsupervised RL}
	\label{alg:framework}
	\begin{algorithmic}[1]
		\State \textbf{Input:} BMDP $\cM$; $\cU\cL\cO$; $(\varepsilon, \delta)$-correct episodic no-regret algorithm $\sA$; batch size $B>0$; $\varepsilon\in(0,1)$; $\delta\in(0,1)$; $N:=\lceil\log(2/\delta)/2\rceil$; $L:=\lceil9H^2/(2\varepsilon^2)\log(2N/\delta)\rceil$.
		\For {$n=1$ {\bfseries to } $N$ }
		\State Clear the memory of $\sA$ and restart;
		\For {episode $k=1$ {\bfseries to } $K$ }
		\State Obtain $\pi^k$ from $\sA$;
		\State Obtain a trajectory: 
		$\tau^k, f_{[H+1]}^{k}\gets \tsr(\ulo, \pi^k, B)$;\label{line:calltsr}
		\State Update the algorithm:
		$\sA\gets \tau^k$; %$\pi^{k+1}\gets\sA(\tau^k)$;\label{line:alg1_a}
		\EndFor
		\State Obtain $\pi^{K+1}$ from $\sA$;
		\State Finalize the decoding functions: $\tau^{K+1}, f_{[H+1]}^{K+1}\gets\texttt{TSR}(\cU\cL\cO, \pi^{K+1}, B)$; %, \cD^{K}, f_{[H+1]}^K)$;
		\State Construct a policy for $\cM: \phi^n\gets\pi^{K+1}\circ f_{[H+1]}^{K+1}$.
		\EndFor
		\State Run each $\phi^n~ (n\in[N])$ for $L$ episodes and get the average rewards per episode $\bar{V}_1^{\phi^n}$.
		\State Output a policy $\phi\in\argmax_{\phi\in\phi^{[N]}} \bar{V}_1^{\phi}$.
	\end{algorithmic}
\end{algorithm}

\subsection{Unsupervised Learning Oracle and No-regret Tabular RL Algorithm}
In this paper, we consider RL on a BMDP. The goal is to find a near-optimal policy with sample complexity polynomial to the cardinality of the latent state space. We assume no knowledge of $\cP$, $r$, and $f_{[H+1]}$, but the access to an unsupervised learning oracle $\cU\cL\cO$ and an ($\varepsilon, \delta$)-correct episodic no-regret algorithm. We give the definitions below.

\begin{definition}[Unsupervised Learning Oracle $\cU\cL\cO$]
	\label{def:ulo}
	There exists a function $g(n, \delta)$ such that for any fixed $\delta > 0$, $\lim_{n\rightarrow\infty}g(n,\delta)=0$.
	Given a distribution $\mu$ over $\cS$, and $n$ samples from $\sum_{s \in \cS}q(\cdot | s)\mu(s)$, with probability at least $1-\delta$, 
	%over $n$ training data, 
	we can find a function $\hat{f} : \cX \rightarrow \cS$ such that $$\mathbb{P}_{s \sim \mu, x \sim q(\cdot | s)} \big(\hat{f}(x) = \alpha(s) \big) \ge 1- g(n,\delta)$$ for some 
	%fixed 
	unknown permutation $\alpha:\cS\rightarrow\cS$.
\end{definition}
In Definition \ref{def:ulo}, suppose $f$ is the true decoding function, i.e., $\mathbb{P}_{s \sim \mu, x \sim q(\cdot | s)} \big(f(x) = s\big) = 1$. We refer to the permutation $\alpha$ as a \emph{good permutation} between $f$ and $\hat{f}$.
Given $g(n,\delta)$, we define $g^{-1}(\epsilon, \delta) := \min\{N~\vert~ \text{for all } n>N, g(n,\delta)<\epsilon\}.$
Since $\lim_{n\rightarrow\infty}g(n,\delta)=0$, $g^{-1}(\epsilon,\delta)$ is well-defined. We assume that $g^{-1}(\epsilon, \delta)$ is a polynomial in terms of $1/\epsilon,\log(\delta^{-1})$ and possibly problem-dependent parameters.

This definition is motivated by \cite{tang2017exploration} in which authors use auto-encoder and SimHash~\citep{charikar2002similarity} to construct the decoding function and they use this UCB-based approach on top of the decoding function to guide exploration.
It is still an open problem to obtain a sample complexity analysis for auto-encoder. Let alone the composition with SimHash.
Nevertheless, in Appendix \ref{app:example}, we give several examples of $\ulo$ with theoretical guarantees.
Furthermore, once we have an analysis of auto-encoder and we can plug-in that into our framework effortlessly.

%\fei{add some remarks on the realizability and necessity of ULO.}

\begin{definition}[($\varepsilon, \delta$)-correct No-regret Algorithm]\label{def:correct} Let $\varepsilon>0$ and $\delta>0$. $\sA$ is an $(\varepsilon,\delta)$-correct no-regret algorithm if for any  MDP $\cM':=(\cS, \cA, \cP, r, H)$ with the initial state $s_1$, $\sA$
\begin{itemize*}
    \item runs for at most $\poly(|\cS|,|\cA|,H,1/\varepsilon,\log(1/\delta))$ episodes (the sample complexity of $\sA$);
    \item proposes a policy $\pi^{k}$ at the beginning of episode $k$ and collects a trajectory of $\cM'$ following $\pi^{k}$;
    \item outputs a policy $\pi$ at the end such that with probability at least $1-\delta$, $\pi$ is $\varepsilon$-optimal.
\end{itemize*}
\end{definition}

Definition \ref{def:correct} simply describes tabular RL algorithms that have polynomial sample complexity guarantees for episodic MDPs. Instances are vivid in literature (see Section~\ref{sec:rel}). 
%\cite{kearns2002near,strehl2006pac, azar2017minimax, agrawal2017posterior, jaksch2010near,  jin2018q}.

\subsection{A Unified Framework}
With a $\ulo$ and an ($\varepsilon, \delta$)-correct no-regret algorithm $\sA$, we propose a unified framework in Algorithm \ref{alg:framework}. Note that we use $\ulo$ as a black-box oracle for abstraction and generality.
%The main challenge is that the true states are unobservable. 
%If directly regarding the BMDP as a giant MDP, the dimension is significantly large. 
%To avoid this, our solution is to unsupervised learn decoding functions with observation samples, then conduct exploration in the low-dimensional latent space. 
For each episode, we combine the policy $\pi$ proposed by $\sA$ for the underlying MDP together with certain decoding functions $\hat{f}_{[H+1]}$\footnote{For the convenience of analysis, we learn decoding functions for each level separately. In practice, observation data can be mixed up among all levels and used to train one decoding function for all levels.} to generate a policy $\pi\circ \hat{f}_{[H+1]}$ for the BMDP. 
Then we collect observation samples using $\pi\circ \hat{f}_{[H+1]}$ and all previously generated policies over the BMDP.
As more samples are collected, we refine the decoding functions via $\ulo$.
Once the sample number is enough, a trajectory of $\pi\circ \hat{f}_{[H+1]}$ is as if obtained using the true decoding functions (up to some permutations).
 Therefore, we successfully simulate the process of running $\pi$ directly on the underlying MDP.
We then proceed to the next episode with the latest decoding functions and repeat the above procedure until $\sA$ halts. 
Note that this procedure is essentially what practitioners use in \cite{tang2017exploration,bellemare2016unifying}  as we have discussed in the beginning.

We now describe our algorithm in more detail.
Suppose the algorithm $\sA$ runs for $K$ episodes. 
At the beginning of each episode $k\in[K]$, $\sA$ proposes a policy $\pi^k:[H]\times\cS\rightarrow\cA$ for the underlying MDP.
Then we use the Trajectory Sampling Routine ($\texttt{TSR}$) to generate a trajectory $\tau^k$ using $\pi^k$ and then feed $\tau^k$ to $\sA$.
After $K$ episodes, we obtain a policy $\pi^{K+1}$ from $\sA$ and a set of decoding functions $f^{K+1}_{[H+1]}$ from $\texttt{TSR}$. 
We then construct a policy for the BMDP as $\pi^{K+1}\circ f^{K+1}_{[H+1]}$. 
We repeat this process for $N$ times for making sure our algorithm succeeds with high probability.
%We claim that with high probability, $\phi$ is a near-optimal policy for the BMDP.

\begin{algorithm}[t]
  \caption{Trajectory Sampling Routine \texttt{TSR} ($\cU\cL\cO, \pi, B$)}
  \label{alg:tsr}
  \begin{algorithmic}[1]
  \State \textbf{Input:} $\cU\cL\cO$; a policy $\pi:[H]\times\cS\rightarrow \cA$; episode index $k$; batch size $B>0$; $\epsilon\in(0,1)$; $\delta_1\in(0,1)$; $J:=(H+1)|\cS|+1$.
  \State \textbf{Data:}
  \vspace{-0.2cm}
  \begin{itemize}
    \setlength\itemsep{0em}
    %\item all training data $\cD$ from previous runs; 
    \item a policy set $\Pi$;
    \item label standard data $\cZ:=\{\cZ_{1}, \cZ_{2}, \ldots \cZ_{H+1}\}$, 
    $\cZ_h:=\{\cD_{h,s_1}, \cD_{h,s_2},\ldots\}$; 
    \item present decoding functions $f^{0}_{[H+1]}$;
    
  \end{itemize}
  \vspace{-0.1cm}
  %\State Let $J\gets|\cS| +1$;
  \For {$i=1$ {\bfseries to } $J$ }
    %\State Generate $B$ trajectories of training data ${\cD'}$ and $B$ trajectories of testing data $\cD''$ with $\pi\circ {f}_{[H+1]}^{i-1}$; \label{line:tsr_generate}
    %\State Combine data: $\cD\gets\cD\cup{\cD'}$;
    \State Combine policy: $\Pi\gets \Pi \cup \{\pi\circ f^{i-1}_{[H+1]}\}$;
    \State Generate $((k-1)J+i)\cdot B$ trajectories of training data $\cD$ with $\unif(\Pi)$;\label{line:tsr_generate1}
    \State Generate $B$ trajectories of testing data $\cD''$ with $\pi\circ {f}_{[H+1]}^{i-1}$.\label{line:tsr_generate2}
    \State Train with $\ulo$: $\tilde{f}_{[H+1]}^{i}\gets\ulo(\cD)$;\label{line:trainulo}
    \State Match labels:
    $f^{i}_{[H+1]}\gets \text{FixLabel}(\tilde{f}_{[H+1]}^{i}, \cZ)$; \label{line:tsr_label}
    \For {$h\in[H+1]$}
    \State 
    Let $\cD''_{h,s} := \{x\in \cD''_{h} : f_h^{i}(x) =s, s\in \cS_h\}$;\label{line:D''s}
    \State Update label standard set: if $\cD_{h,s}\not\in \cZ_h$ and $|\cD''_{h,s}|\ge 3\epsilon \cdot B\log(\delta_1^{-1})$, then let $\cZ_h\gets\cZ_h\cup \{\cD''_{h,s}\}$\label{line:update}
    \EndFor\label{line:tsr_label_end}
  \EndFor
   \State Run $\pi\circ f_{[H+1]}^{J}$ to obtain a trajectory $\tau$;
   \State Renew $f^{0}_{[H+1]}\gets f_{[H+1]}^{J}$;
   \State \textbf{Output:} $\tau, f_{[H+1]}^{J}$.
\end{algorithmic}
\end{algorithm}
\begin{algorithm}[t]
	\caption{FixLabel($\tilde{f}_{[H+1]}, \cZ$)}
	%, \cD^{-1}, \cT, f^{0}_{[H+1]}, B$)}
	\label{alg:label}
	\begin{algorithmic}[1]
		\State \textbf{Input:} decoding functions $\tilde{f}_{[H+1]}$; a set of label standard data $\cZ:=\{\cZ_1, \cZ_2, \cdots, \cZ_{H+1}\}$, $\cZ_h:=\{\cD_{h,s_1}, \cD_{h,s_2}, \ldots\}$.
		%\For {$h\in[H+1]$}
		\For {every $\cD_{h,s}$ in $\cZ$}
		%\For {$\cD_{h,s}\in \cZ_h$ }
		\If{$s\in \cS_h$ and $|\{x\in \cD_{h,s}: \tilde{f}_{h}(x) = s'\}|> 3/5|\cD_{h,s}|$}
		\State Swap the output of $s'$ with $s$ in $\tilde{f}_h$;
		\EndIf
		%\EndFor
		\EndFor
		\State{\bfseries Output:} $\tilde{f}_{[H+1]}$
	\end{algorithmic}
\end{algorithm}

\cut{
\begin{algorithm}[t]
	\caption{Trajectory Sampling Routine \texttt{TSR} ($\cU\cL\cO, \pi, B$)}
	\label{alg:tsr}
	\begin{algorithmic}[1]
		\State \textbf{Input:} $\cU\cL\cO$; a tabular RL policy $\pi$; batch size $B$; $\epsilon\in(0,1)$; $\delta_1\in(0,1)$; $J:=(H+1)|\cS|+1$.
		\State \textbf{Data:} all training data $\cD$ from previous runs;  label standard data $\cZ:=\{\cZ_{1}, \cZ_{2}, \ldots \cZ_{H+1}\}$, 
		$\cZ_h:=\{\cD_{h,s_1}, \cD_{h,s_2},\ldots\}$; the latest decoding functions $f^{0}_{[H+1]}$.
		%		\begin{itemize}
		%			\setlength\itemsep{-0.1em}
		%			\item all training data $\cD$ from previous runs; 
		%			\item label standard data $\cZ:=\{\cZ_{1}, \cZ_{2}, \ldots \cZ_{H+1}\}$, 
		%			$\cZ_h:=\{\cD_{h,s_1}, \cD_{h,s_2},\ldots\}$;
		%			\item the latest decoding functions $f^{0}_{[H+1]}$.
		%		\end{itemize}
		%		\vspace{-0.1cm}
		%\State Let $J\gets|\cS| +1$;
		\For {$i=1$ {\bfseries to } $J$ }
		\State Run $\pi\circ {f}_{[H+1]}^{i-1}$ in the BMDP environment to get $B$ trajectories as training data ${\cD'}$ and another $B$ trajectories as testing data $\cD''$; \label{line:tsr_generate}
		\State Combine training data: $\cD\gets\cD\cup{\cD'}$;
		\State Train with $\ulo$: $\tilde{f}_{[H+1]}^{i}\gets\ulo(\cD)$;\label{line:trainulo}
		\State Match labels:
		$f^{i}_{[H+1]}\gets \text{FixLabel}(\tilde{f}_{[H+1]}^{i}, \cZ)$; \label{line:tsr_label}
		\For {$h\in[H+1]$}
		\State 
		Let $\cD''_{h,s} := \{x\in \cD''_{h} : f_h^{i}(x) =s, s\in \cS_h\}$;\label{line:D''s}
		\State Update $\cZ$: if $\cD_{h,s}\not\in \cZ_h$ and $|\cD''_{h,s}|\ge 3\epsilon \cdot B\log(\delta_1^{-1})$, then let $\cZ_h\gets\cZ_h\cup \{\cD''_{h,s}\}$\label{line:update}
		\EndFor
		\EndFor
		\State Run $\pi\circ f_{[H+1]}^{J}$ to obtain a trajectory $\tau$;
		\State Renew $f^{0}_{[H+1]}\gets f_{[H+1]}^{J}$;
		\State \textbf{Output:} $\tau, f_{[H+1]}^{J}$.
	\end{algorithmic}
\end{algorithm}
}

The detailed description of $\texttt{TSR}$ is displayed in Algorithm \ref{alg:tsr}. 
We here briefly explain the idea. To distinguish between episodes, with input policy $\pi^k$ (Line \ref{line:calltsr} Algorithm \ref{alg:framework}), we add the episode index $k$ as superscripts to $\pi$ and $f_{[H+1]}$ in $\texttt{TSR}$.
We maintain a policy set in memory and initialize it as an empty set at the beginning of Algorithm \ref{alg:framework}.
Note that, at each episode, our goal is to simulate a trajectory of $\pi$ running on the underlying MDP. 
%As discussed in Section \ref{sec:setting}, if we known the true decoding functions $f_{[H+1]}$, then running the induced policy $\pi^k\circ f_{[H+1]}$ on BMDP is equivalent to running $\pi^k$ on the underlying MDP. Now, without $f_{[H+1]}$, we need to \emph{learn} the correct labels.
$\texttt{TSR}$ achieves this in an iterative fashion: it starts with the input policy $\pi^k$ and the latest-learned decoding functions $f^{k,0}_{[H+1]}:=f^{k-1, J}_{[H+1]}$; for each iteration $i$, it first adds the policy $\pi^k\circ f^{k,i-1}_{[H+1]}$ in $\Pi$ and then plays $\unif(\Pi)$ to collect a set of observation trajectories (i.e., each trajectory is generated by first uniformly randomly selecting a policy from $\Pi$ and then running it in the BMDP);\footnote{This resampling over all previous policies is mainly for the convenience of analysis. It can be replaced using previous data but requires more refined analysis.} then updates $f^{k, i-1}_{[H+1]}$ to $\tilde{f}^{k,i}_{[H+1]}$ by running $\ulo$ on these collected observations.
%for each iteration $i$, it collects a batch of observations by playing $\pi^k\circ f^{k, i-1}_{[H+1]}$; then updates $f^{k, i-1}_{[H+1]}$ to $\tilde{f}^{k,i}_{[H+1]}$ by running $\ulo$ on all existing observations.
Note that $\ulo$ may output labels inconsistent with previously trained decoding functions. We further match labels of $\tilde{f}^{k,i}_{[H+1]}$ with the former ones by calling the FixLabel routine (Algorithm \ref{alg:label}).
%FixLabel is fairly straightforward so we defer it to appendix.
To accomplish the label matching process, we cache a set $\cZ$ in memory which stores observation examples $\cD_{h,s}$ for each state $s$ and each level $h$. $\cZ$ is initialized as an empty set and gradually grows. Whenever we confirm a new label, we add the corresponding observation examples to $\cZ$ (Line \ref{line:update} Algorithm \ref{alg:tsr}). Then for later learned decoding functions, they can use this standard set to correspondingly swap their labels and match with previous functions. After the matching step, we get $f^{k,i}_{[H+1]}$. 
Continuously running for $J$ iterations, we stop and use $\pi^k\circ f^{k,J}_{[H+1]}$ to obtain a trajectory. 

We now present our main theoretical result.
\begin{theorem}\label{thm:main}
	Suppose in Definition \ref{def:ulo}, $g^{-1}(\epsilon, \delta_1)=\poly(|\cS|, 1/\epsilon, \log(\delta_1^{-1}))$ for any $\epsilon, \delta_1\in(0,1)$ and $\sA$ is $(\varepsilon, \delta_2)$-correct with sample complexity $\poly\left(|\cS|,|\cA|,H,1/\varepsilon,\log\left(\delta_2^{-1}\right)\right)$ for any $\varepsilon,\delta_2 \in (0,1)$.
	Then  Algorithm~\ref{alg:framework}
	outputs a policy $\phi$ such that
	with probability at least $1-\delta$, $\phi$ is an $\varepsilon$-optimal policy for the BMDP, using at most $\poly\left(|\cS|,|\cA|,H,1/\varepsilon, \log(\delta^{-1}) \right)$ trajectories.
\end{theorem} 
We defer the proof to Appendix \ref{app:proof}. Theorem \ref{thm:main} formally justifies what we claimed in Section~\ref{sec:intro} that as long as the sample complexity of $\ulo$ is polynomial and $\sA$ is a no-regret tabular RL algorithm, polynomial number of trajectories suffices to find a near-optimal policy.
To our knowledge, this is the first result that proves unsupervised learning can guide exploration in RL problems with a huge observation space.
%the empirical paradigm used in \cite{tang2017exploration,bellemare2016unifying}.

\cut{\section{Examples of Unsupervised Learning Oracle}\label{sec:example}
In this section, we give some examples of $\ulo$. First notice that %observations are generated by first sampling a state $s$ following certain finite distribution $\mu(\cdot)$ and then sampling an observation $x$ following a distribution $q(\cdot|s)$ determined by the selected state. Such a generation mechanism
the generation process of $\ulo$ is termed as the mixture model in statistics \citep{mclachlan1988mixture, mclachlan2004finite}, which has a wide range of applications (see e.g., \citet{bouguila2020mixture}).
%from image segmentation to genomic analysis
%Mixture models have been generally adopted in the study of clustering and unsupervised learning and may quantitative results are available.Therefore, we have vivid options for $\ulo$.In the sequel, 
We list examples of mixture models and some algorithms as candidates of $\ulo$.

\paragraph{Gaussian Mixture Models (GMM)} %The first class we consider is the Gaussian Mixture Model (GMM).
%GMM is the most intensively studied subarea of mixture models.  
%Suppose our states and observations are points in $\RR^d$. 
In GMM, $q(\cdot|s)=\cN(s,\sigma_s^2)$, i.e., observations are hidden states plus Gaussian noise.\footnote{To make the model satisfy the disjoint block assumption in the definition of BMDP, we need some truncation of the Gaussian noise so that each observation only corresponds to a unique hidden state.}
%This is a natural assumption considering that noised data is ubiquitous. 
When the noises are (truncated) Gaussian, under certain conditions, e.g. states are well-separated, we are able to identify the latent states with high accuracy. 
%For example, one approach is the separation-based method. The goal of separation-based methods is to geometrically getting close to the means under assumptions that means are well-separated\footnote{One commonly adopted separation notion is $c$-separated, i.e., $\forall s\neq s' \in \cS, \|s-s'\|\geq c\max\{\sigma_s,\sigma_{s'}\}\sqrt{d}.$}.
A series of works \citep{sanjeev2001learning, vempala2004spectral, achlioptas2005spectral, dasgupta2000two, regev2017learning} proposed algorithms that can be served as $\ulo$.
%proposed algorithms that can return arbitrarily close estimations of the true means with a polynomial number of samples.
%By selecting any of these algorithms as a $\ulo$, we are able to predict the latent state with high probability.

\paragraph{Bernoulli Mixture Models (BMM)} BMM is considered in binary image processing \citep{juan2004bernoulli} and texts classification \citep{juan2002use}. In BMM, every observation is a point in $\{0,1\}^{d}$. A true state determines a frequency vector.
%$p^s\in[0,1]^d$ such that $x_i\sim \text{Bernoulli}(p^s_i), i\in[d]$ and therefore, $q(x|s)=\Pi_{i=1}^d (p^s_i)^{x_i}(1-p^s_i)^{1-x_i}$. 
In \cite{najafi2020reliable}, the authors proposed a reliable clustering algorithm for BMM data with polynomial sample complexity guarantee. 
%(see Theorem 1 in \cite{najafi2020reliable} for details).
%and firstly provided PAC analysis for BMM clustering: suppose $d>\poly(\epsilon)$ and the frequency vectors $\{p^s\}_{s\in\cS}$ satisfy certain separability condition, with probability at least $1-\delta$, the output cluster is correct up to an $\epsilon$ fraction of mis-clustering with $\poly(1/\epsilon, 1/\delta)$ many samples .

\paragraph{Subspace Clustering}
In some applications, each state is a set of vectors and observations lie in the spanned subspace. 
%where each vector is in $\RR^d$ and $k$ is much smaller than $d$. Observations of $s$ are points lying in the subspace spanned by these vectors and $q(\cdot|s)$ is a probability measure over the subspace. 
Suppose for different states, the basis vectors differ under certain metric, then recovering the latent state is equivalent to subspace clustering. Subspace clustering 
%is an effective technique to deal with high-dimensional data \citep{parsons2004subspace}. 
%Compared with feature selection which removes irrelevant dimensions by analyzing the entire dataset, subspace clustering algorithms localize the search for relevant dimensions and are able to uncover clusters that exist in multiple subspaces \cite{parsons2004subspace}. 
has a variety of applications include face clustering, community clustering, and DNA sequence analysis \citep{wallace2015application, vidal2011subspace, elhamifar2013sparse}. Proper algorithms for $\ulo$ can be found in e.g., \citep{wang2013provable, soltanolkotabi2014robust}. %In \cite{wang2013provable}, if the data points are i.i.d. sampled from the unit sphere embedded in each subspace and samples from one subspace are \emph{separable}\footnote{Definition 3 in \cite{wang2013provable}} against samples from other subspaces, with $\poly(k,d,1/\delta)$ samples, the proposed algorithm returns correct clusters, i.e. samples from the same subspace are clustered together, with probability at least $1-\delta$.
%If the observations are noisy, then one can apply the algorithm in \cite{soltanolkotabi2014robust} which provably recovers clustering with high accuracy under the condition that any two subspaces are not close to each other\footnote{Definition 1.1 and 1.2 in \cite{soltanolkotabi2014robust}.}.

In addition to the aforementioned models, other reasonable settings are Categorical Mixture Models \citep{bontemps2013clustering}, Poisson Mixture Models \citep{li2006two}, Dirichlet Mixture Models \citep{dahl2006model} and so on.

}

\section{Numerical Experiments}\label{sec:test}
\begin{figure}[t]
	\centering
	\includegraphics[width=.9\textwidth]{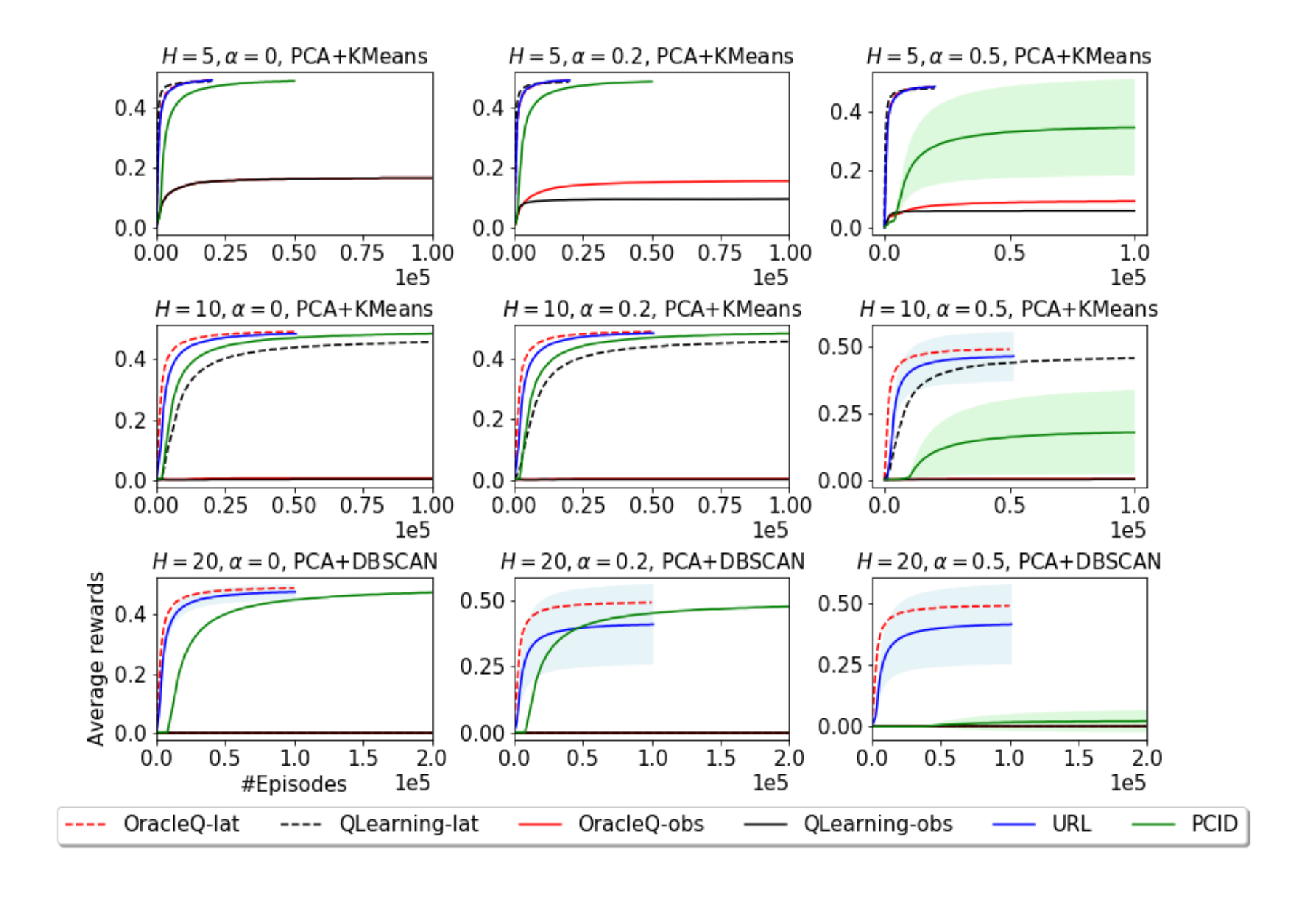}
    %\vspace{-0.3cm}
    	\caption{Performances for LockBernoulli. All lines are mean values of $50$ tests and the shaded areas depict the one standard deviations.}%The title for each subfigure records the length of the horizon, switch parameter $\alpha$ in actions, and the unsupervised learning method we apply for URL. }
	\label{fig:ber}
	%\vspace{-0.4cm}
\end{figure}

\begin{figure}[t]
	\centering
	\includegraphics[width=.9\textwidth]{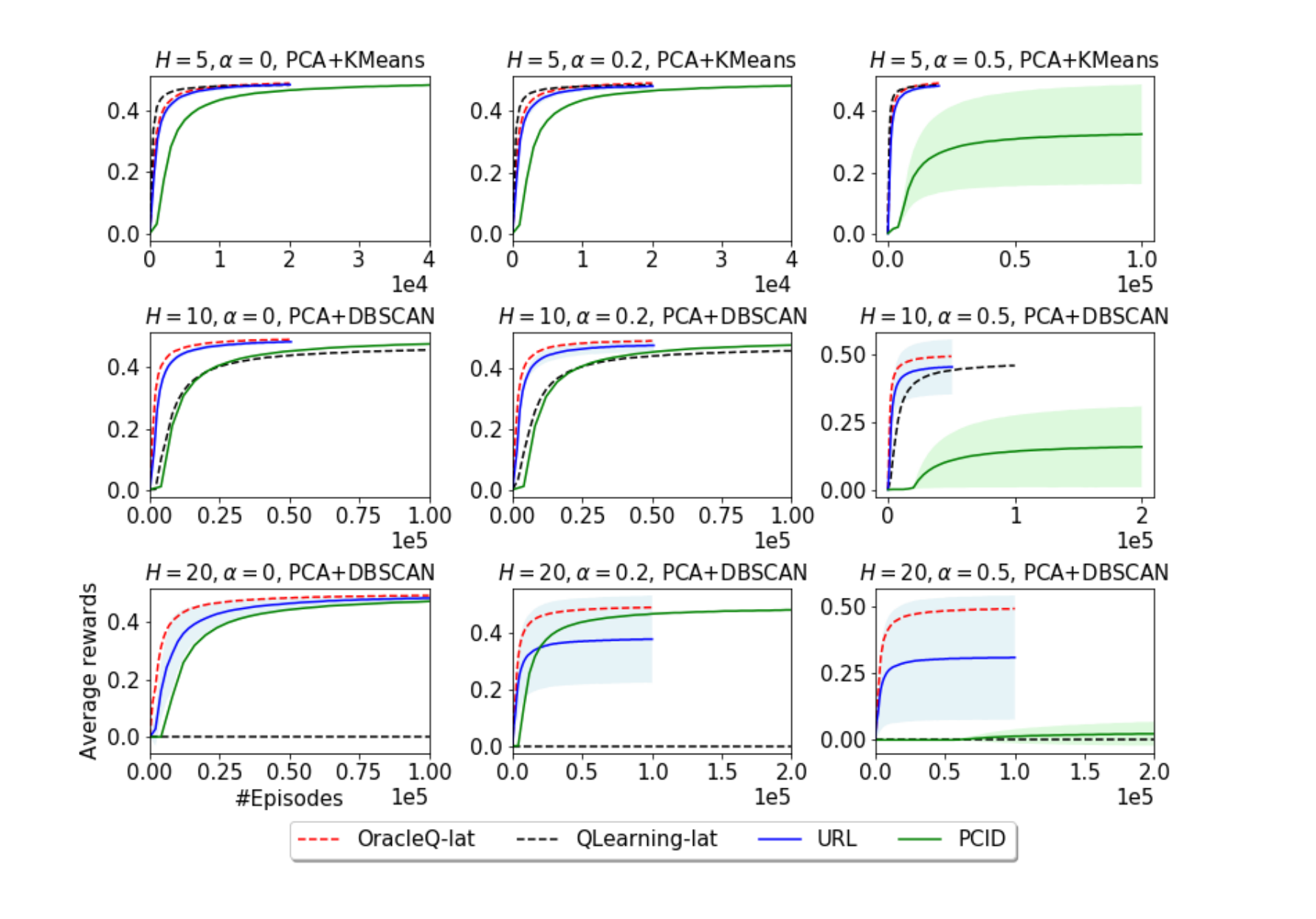}
	%\vspace{-0.3cm}
	\caption{Performances for LockGaussian, $\sigma=0.1$. All lines are mean values of $50$ tests and the shaded areas depict the one standard deviations. OracleQ-lat and QLearning-lat have direct access to the latent states, which are not for practical use. URL and PCID only have access to the observations. 
		OracleQ-obs and QLearning-obs are omitted due to infinitely many observations.
		 }
	%\vspace{-0.5cm}
	\label{fig:gau1}
\end{figure}

\begin{figure}[t]
	\centering
	\includegraphics[width=.9\textwidth]{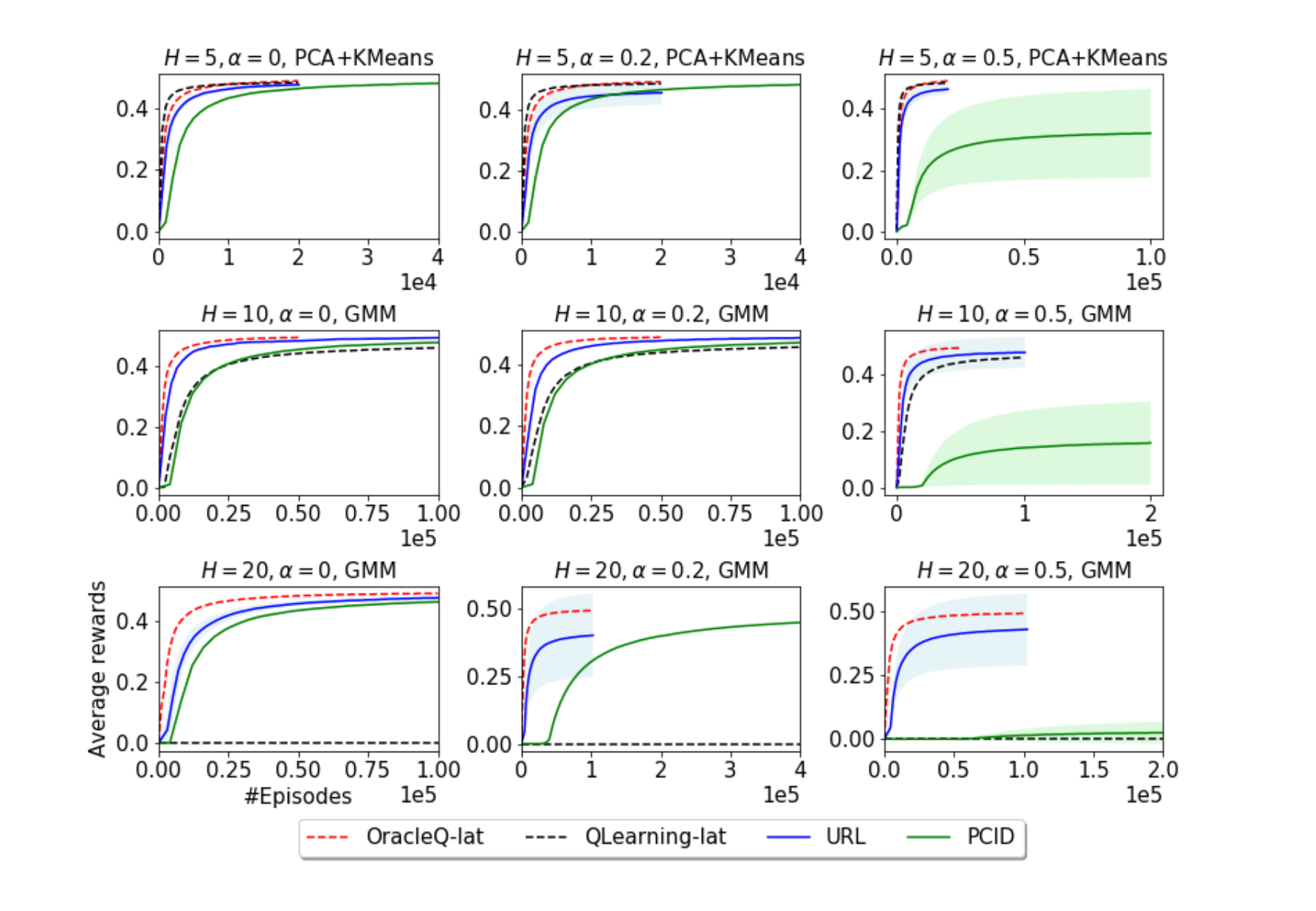}
	%\vspace{-0.5cm}
	\caption{Performances for LockGaussian, $\sigma=0.2$. All lines are mean values of $50$ tests and the shaded areas depict the one standard deviations. OracleQ-lat and QLearning-lat have direct access to the latent states, which are not for practical use. URL and PCID only have access to the observations. OracleQ-obs and QLearning-obs are omitted due to infinitely many observations.}
	%	\vspace{-0.8cm}
	\label{fig:gau2}
\end{figure}

In this section we conduct experiments to demonstrate the effectiveness of our framework. Our code is available at \href{https://github.com/FlorenceFeng/StateDecoding}{\texttt{https://github.com/FlorenceFeng/StateDecoding}}.

\paragraph{Environments}
We conduct experiments in two environments: LockBernoulli and LockGaussian.
These environments are also studied in \cite{du2019provably}, which are designed to be hard for exploration.
Both environments have the same latent state structure with $H$ levels, 3 states per level and 4 actions. At level $h$, from states $s_{1,h}$ and $s_{2,h}$ one action leads with probability $1-\alpha$ to $s_{1,h+1}$ and with probability $\alpha$ to $s_{2,h+1}$, another has the flipped behavior, and the remaining two lead to $s_{3,h+1}$. All actions from $s_{3,h}$ lead to $s_{3,h+1}$. Non-zero reward is only achievable if the agent can reach $s_{1,H+1}$ or $s_{2,H+1}$ and the reward follows Bernoulli(0.5). Action labels are randomly assigned at the beginning of each time of training. We consider three values of $\alpha$: 0, 0.2, and 0.5.

In LockBernoulli, the observation space is $\{0,1\}^{H+3}$ where the first 3 coordinates are reserved for the one-hot encoding of the latent state and the last $H$ coordinates are drawn i.i.d from Bernoulli(0.5). LockBernoulli meets our requirements as a BMDP. In LockGaussian, the observation space is $\RR^{H+3}$. Every observation is constructed by first letting the first three coordinates be the one-hot encoding of the latent state, then adding i.i.d Gaussian noises $\cN(0,\sigma^2)$ to all $H+3$ coordinates. We consider $\sigma=0.1$ and $0.2$. LockGaussian is not a BMDP. We use this environment to evaluate the robustness of our method to violated assumptions.

The environments are designed to be hard for exploration. There are in total $4^H$ choices of actions of one episode, but only $2^H$ of them lead to non-zero reward in the end. So random exploration requires exponentially many trajectories. Also, with a larger $H$, the difficulty of learning accurate decoding functions increases and makes exploration with observations a more challenging task.

\paragraph{Algorithms and Hyperparameters}
We compare 4 algorithms: OracleQ \citep{jin2019provably}; QLearning, the tabular Q-Learning with $\epsilon$-greedy exploration; URL, our method; and PCID \citep{du2019provably}. For OracleQ and QLearning, there are two implementations: 1. they directly see the latent states (OracleQ-lat and QLearning-lat); 2. only see observations (OracleQ-obs and QLearning-obs). For URL and PCID, only observations are available. OracleQ-lat and QLearning-lat are served as a near-optimal skyline and a sanity-check baseline to measure the efficiency of observation-only algorithms. OracleQ-obs and QLearning-obs are only tested in LockBernoulli since there are infinitely many observations in LockGaussian. For URL, we use OracleQ as the tabular RL algorithm. Details about hyperparameters and unsupervised learning oracles in URL can be found in Appendix \ref{app:experiment}.

\paragraph{Results}
The results are presented in Figure~\ref{fig:ber},~\ref{fig:gau1}, and~\ref{fig:gau2}. $x$-axis is the number of training trajectories and $y$-axis is average reward.
All lines are mean values of $50$ tests and the shaded areas depict the one standard deviations. The title for each subfigure records problem parameters and the unsupervised learning method we apply for URL.
%The title for each subfigure records the length of the horizon, switch parameter $\alpha$ in actions, and the unsupervised learning method we apply for URL. 
In LockBernoulli, OracleQ-obs and QLearning-obs are far from being optimal even for small-horizon cases. URL is mostly as good as the skyline (OracleQ-lat) and much better than the baseline (QLearning-lat) especially when $H=20$. URL outperforms PCID in most cases. When $H=20$, we observe a probability of $80\%$ that URL returns near-optimal values for $\alpha=0.2$ and $0.5$. In LockGaussian, 
%similar performances are observed.
%, OracleQ-obs and QLearning-obs are omitted due to infinitely many observations. Similarly, in most cases URL outperforms PCID. 
for $H= 20$, we observe a probability of $>75\%$ that URL returns a near-optimal policy for $\alpha= 0.2$ and 0.5. 

%In summary, we preliminarily show the effectiveness and flexibility of our framework in terms of sample complexity and options of unsupervised learning methods. By selecting proper unsupervised learning approaches (especially when prior knowledge is available), we believe our framework is competitive and practical for RL with rich observations.

\section{Conclusion}\label{sec:con}
The current paper gave a general framework that turns an unsupervised learning algorithm and a no-regret tabular RL algorithm into an algorithm for RL problems with huge observation spaces.
We provided theoretical analysis to show it is provably efficient.
We also conducted numerical experiments to show the effectiveness of our framework in practice.
%This result complements empirical findings that unsupervised learning can guide exploration.
An interesting future theoretical direction is to characterize the optimal sample complexity under our assumptions.

\section*{Broader Impact}
Our research broadens our understanding on the use of unsupervised learning for RL, which in turn can help researchers design new principled algorithms and incorporate safety considerations for more complicated problems.

We do not believe that the results in this work will cause any ethical issue, or put anyone at a disadvantage in the society.
\begin{ack}
Fei Feng was supported by AFOSR MURI FA9550-18-10502 and ONR N0001417121.
This work was done while Simon S. Du was at the Institute for Advanced Study and he was supported by NSF grant DMS-1638352 and the Infosys Membership.
\end{ack}

\bibliographystyle{apalike}
\bibliography{references}  

\newpage
\appendix

\cut{\begin{algorithm}[t]
	\caption{A Unified Framework for Unsupervised RL}
	\label{alg:framework}
	\begin{algorithmic}[1]
		\State \textbf{Input:} BMDP $\cM$; $\cU\cL\cO$; $(\varepsilon, \delta)$-correct episodic no-regret algorithm $\sA$; batch size $B>0$; $\varepsilon\in(0,1)$; $\delta\in(0,1)$; $N:=\lceil\log(2/\delta)/2\rceil$; $L:=\lceil9H^2/(2\varepsilon^2)\log(2N/\delta)\rceil$.
		\For {$n=1$ {\bfseries to } $N$ }
		\State Clear the memory of $\sA$ and restart;
		\For {episode $k=1$ {\bfseries to } $K$ }
		\State Obtain $\pi^k$ from $\sA$;
		\State Obtain a trajectory: 
		$\tau^k, f_{[H+1]}^{k}\gets \tsr(\ulo, \pi^k, B)$;\label{line:calltsr}
		\State Update the algorithm:
		$\sA\gets \tau^k$; %$\pi^{k+1}\gets\sA(\tau^k)$;\label{line:alg1_a}
		\EndFor
		\State Obtain $\pi^{K+1}$ from $\sA$;
		\State Finalize the decoding functions: $\tau^{K+1}, f_{[H+1]}^{K+1}\gets\texttt{TSR}(\cU\cL\cO, \pi^{K+1}, B)$; %, \cD^{K}, f_{[H+1]}^K)$;
		\State Construct a policy for $\cM: \phi^n\gets\pi^{K+1}\circ f_{[H+1]}^{K+1}$.
		\EndFor
		\State Run each $\phi^n~ (n\in[N])$ for $L$ episodes and get the average rewards per episode $\bar{V}_1^{\phi^n}$.
		\State Output a policy $\phi\in\argmax_{\phi\in\phi^{[N]}} \bar{V}_1^{\phi}$.
	\end{algorithmic}
\end{algorithm}
\begin{algorithm}[t]
	\caption{Trajectory Sampling Routine \texttt{TSR} ($\cU\cL\cO, \pi, B$)}
	\label{alg:tsr}
	\begin{algorithmic}[1]
		\State \textbf{Input:} $\cU\cL\cO$; a tabular RL policy $\pi$; batch size $B$; $\epsilon\in(0,1)$; $\delta_1\in(0,1)$; $J:=(H+1)|\cS|+1$.
		\State \textbf{Data:} all training data $\cD$ from previous runs;  label standard data $\cZ:=\{\cZ_{1}, \cZ_{2}, \ldots \cZ_{H+1}\}$, 
		$\cZ_h:=\{\cD_{h,s_1}, \cD_{h,s_2},\ldots\}$; the latest decoding functions $f^{0}_{[H+1]}$.
		%		\begin{itemize}
		%			\setlength\itemsep{-0.1em}
		%			\item all training data $\cD$ from previous runs; 
		%			\item label standard data $\cZ:=\{\cZ_{1}, \cZ_{2}, \ldots \cZ_{H+1}\}$, 
		%			$\cZ_h:=\{\cD_{h,s_1}, \cD_{h,s_2},\ldots\}$;
		%			\item the latest decoding functions $f^{0}_{[H+1]}$.
		%		\end{itemize}
		%		\vspace{-0.1cm}
		%\State Let $J\gets|\cS| +1$;
		\For {$i=1$ {\bfseries to } $J$ }
		\State Run $\pi\circ {f}_{[H+1]}^{i-1}$ in the BMDP environment to get $B$ trajectories as training data ${\cD'}$ and another $B$ trajectories as testing data $\cD''$; \label{line:tsr_generate}
		\State Combine training data: $\cD\gets\cD\cup{\cD'}$;
		\State Train with $\ulo$: $\tilde{f}_{[H+1]}^{i}\gets\ulo(\cD)$;\label{line:trainulo}
		\State Match labels:
		$f^{i}_{[H+1]}\gets \text{FixLabel}(\tilde{f}_{[H+1]}^{i}, \cZ)$; \label{line:tsr_label}
		\For {$h\in[H+1]$}
		\State 
		Let $\cD''_{h,s} := \{x\in \cD''_{h} : f_h^{i}(x) =s, s\in \cS_h\}$;\label{line:D''s}
		\State Update $\cZ$: if $\cD_{h,s}\not\in \cZ_h$ and $|\cD''_{h,s}|\ge 3\epsilon \cdot B\log(\delta_1^{-1})$, then let $\cZ_h\gets\cZ_h\cup \{\cD''_{h,s}\}$\label{line:update}
		\EndFor
		\EndFor
		\State Run $\pi\circ f_{[H+1]}^{J}$ to obtain a trajectory $\tau$;
		\State Renew $f^{0}_{[H+1]}\gets f_{[H+1]}^{J}$;
		\State \textbf{Output:} $\tau, f_{[H+1]}^{J}$.
	\end{algorithmic}
\end{algorithm}
\begin{algorithm}[t]
	\caption{FixLabel($\tilde{f}_{[H+1]}, \cZ$)}
	%, \cD^{-1}, \cT, f^{0}_{[H+1]}, B$)}
	\label{alg:label}
	\begin{algorithmic}[1]
		\State \textbf{Input:} a set of decoding functions $\tilde{f}_{[H+1]}$; a set of label standard data $\cZ:=\{\cZ_1, \cZ_2, \cdots, \cZ_{H+1}\}$, $\cZ_h:=\{\cD_{h,s_1}, \cD_{h,s_2}, \ldots\}$.
		\For {$h\in[H+1]$}
		\For {$\cD_{h,s}\in \cZ_h$ }
		\If{$s\in \cS_h$ and $|\{x\in \cD_{h,s}: \tilde{f}_{h}(x) = s'\}|> 3/5|\cD_{h,s}|$}
		\State Swap the output of $s'$ with $s$ in $\tilde{f}_h$;
		\EndIf
		\EndFor
		\EndFor
		\State{\bfseries Output:} $\tilde{f}_{[H+1]}$
	\end{algorithmic}
\end{algorithm}
}

\section{Proofs for the Main Result}\label{app:proof}
We first give a sketch of the proof. Note that if \tsr~always correctly simulates a trajectory of $\pi^k$ on the underlying MDP, then by the correctness of $\sA$, the output policy of $\sA$ in the end is near-optimal with high probability. If in $\tsr$, $f^{k, J}_{[H+1]}$ decodes states correctly (up to a fixed permutation, with high probability) for every observation generated by playing $\pi^k\circ f^{k, J}_{[H+1]}$, then the obtained trajectory (on $\cS$) is as if obtained with $\pi^k\circ f_{[H+1]}$ which is essentially equal to playing $\pi^k$ on the underlying MDP.
Let us now consider $\pi^k\circ f^{k, i}_{[H+1]}$ for some intermediate iteration $i\in[J]$.
If there are many observations from a previously unseen state, $s$, 
then $\ulo$ guarantees that all the decoding functions in future iterations will be correct with high probability of identifying observations of $s$. Since there are at most $|\cS|$ states to reach for each level following $\pi^k$, after $(H+1)|\cS|$ iterations, \tsr~ is guaranteed to output a set of decoding functions that are with high probability correct under policy $\pi^k$.
With this set of decoding functions, we can simulate a trajectory for $\sA$ as if we know the true latent states.

%In the full description of \tsr ( Algorithm~\ref{alg:tsr}), we denote the decoding functions of episode $k$ and iteration $i$ as $f_{[H+1]}^{k, i}$ (Line~\ref{line:tsr_label}) and $\tilde{f}_{[H+1]}^{k, i}$ as the decoding functions before fixing labels (Line~\ref{line:trainulo}).
For episode $k$, we denote the training dataset $\cD$ generated by running $\unif(\Pi)$ as $\{\cD_{k,i,h}\}_{h=1}^{H+1}$ (Line~\ref{line:tsr_generate1}) and the testing dataset $\cD''$ generated by $\pi^k\circ f^{k,i-1}_{[H+1]}$ as $\{\cD''_{k,i,h}\}_{h=1}^{H+1}$ (Line~\ref{line:tsr_generate2}). The subscript $h$ represents the level of the observations. Furthermore, we denote by $\mu_{k,i,h}(\cdot)$ the distribution over hidden states at level $h$ induced by $\pi^k\circ f^{k,i-1}_{[H+1]}$.
To formally prove the correctness of our framework, we first present the following lemma, showing that whenever some policy
$\pi$ with some decoding functions visits a state $s$ with relatively high probability, all the decoding functions of later iterations will correctly decode the observations from $s$ with high probability.
\begin{lemma}
\label{lem:label}
Suppose for some $s^*\in \cS_h$, $(k,i)$ is the earliest pair such that %$f_{h}^{k, i}$ maps many testing observations to state $s^*$, i.e., 
$\big|\{x\in\cD''_{k,i,h} : f_{h}^{k, i}(x) =\alpha_h(s^*)\}|\ge 3\epsilon\cdot B\log({\delta_1^{-1}})$ and $\{x\in\cD''_{k,i,h} : f_{h}^{k, i}(x) = \alpha_h(s^*)\}$ is added into $\cZ_h$ as $\cD_{h,\alpha_h(s^*)}$ at line \ref{line:update} Algorithm \ref{alg:tsr}, where $\alpha_h$ is a good permutation between $f^{k,i}_h$ and $f_h$. Then for each $(k',i')> (k, i)$ (in lexical order), with probability at least 
$1-\cO(\delta_1)$,
\[
\Pr_{x\sim q(\cdot|s^*)}\big[f^{k',i'}_h(x)
\neq \alpha_h^*(s^*)\big]\le \epsilon
\]
provided $0<\epsilon\log(\delta_1^{-1})\leq 0.1$ and $B\ge B_0$.
Here $B_0$ is some constant to be determined later and $\alpha_h^*$ is some fixed permutation on $\cS_h$.
\end{lemma}
\begin{proof}[Proof of Lemma~\ref{lem:label}]
%Thus, at Line~\ref{line:tsr_generate1} of Algorithm~\ref{alg:tsr}, we have collected $\big((k-1)J+i\big)\cdot B$ training samples $\cD_{k,i,h}$ from the distribution $\unif(\{\mu_{k',i',h}\}_{(k',i')< (k, i)})$ .}
%which is the reaching distribution to level $h$ by playing $\pi^{k}\circ{f}_{[H+1]}^{k, i-1}$. 
%We have also collected $B$ testing samples $\cD''_{k,i,h}$.
For iterations $(k',i')\ge (k, i)$, the function $\tilde{f}_{h}^{k',i'}$ is obtained by applying $\ulo$ on the dataset generated 
by $$\mu':=\unif(\{\mu_{k'',i'',h}\}_{(k'',i'')< (k', i')})$$ and the dataset has size $\big((k'-1)\cdot J+i'\big)\cdot B = \Theta(k'JB)$. 
Thus, with probability at least $1-\delta_1$, for some permutation $\alpha_h'$,
\begin{align}\label{eq:prob_mu'}
\Pr_{s\sim \mu',x\sim q(\cdot|s)}&\big[\tilde{f}_{h}^{k',i'}(x)
\neq \alpha_h'\circ f_h(x)\big]
\le g\big(\Theta(k'JB), \delta_1\big).
\end{align}
%In the later, we will show $K=\poly(|\cS|, |\cA|, H, 1/\varepsilon, \log(\delta^{-1}))$ and $J=\poly(|\cS|, H)$. 
By taking 
\begin{align}\label{eq:choiceB}
B_0:&=\Theta \Big(\frac{g^{-1}(\epsilon^2/(K\cdot J),\delta_1)}{K\cdot J}\Big),
%=\poly(|\cS|,|\cA|,H,1/\varepsilon, 1/\epsilon, \log(\delta^{-1})),
\end{align}
we have when $B\geq B_0$, $g\big(\Theta(k'JB),\delta_1)\leq\epsilon^2/(K\cdot J)$ for all $k'\in[K]$. Later, in Proposition \ref{prop:main}, we will show that $B_0=\poly(|\cS|, |\cA|, H, 1/\varepsilon)$. Now we consider ${f}_{h}^{k,i}$. Since the FixLabel routine (Algorithm \ref{alg:label}) does not change the accuracy ratio, from Equation \eqref{eq:prob_mu'}, it holds with probability at least $1-\delta_1$ that
\begin{align}
    \Pr_{s\sim\mu_{k,i,h},x\sim q(\cdot|s)}[f^{k,i}_h(x)\neq \alpha_h\circ f_h(x)]\leq k\cdot J\cdot g\big(\Theta(kJB), \delta_1\big)\leq \epsilon.
\end{align}
Therefore, by Chernoff bound, with probability at least $1-\cO(\delta_1)$,
\begin{align}
\big|\{x\in \cD''_{k, i, h}: f_h(x)\neq s &\text{ and } {f}_{h}^{k,i}(x) = \alpha_h(s)
\}\big| 
< ~\epsilon\cdot B\log({\delta_1^{-1}}).
\end{align}
Since $\big|\{x\in\cD''_{k,i,h}: f_{h}^{k, i}(x) = \alpha_h(s^*)\}|\ge 3\epsilon\cdot B\log({\delta_1^{-1}})$,
we have that
\begin{align}\label{eq:Ds*}
\big|\{x\in\cD''_{k,i,h}: ~f_h(x)= s^* \text{ and } {f}_{h}^{k,i}(x) = \alpha_h(s^*)\}|
> &~\frac23\cdot \big|\{x\in\cD''_{k,i,h}: f_{h}^{k, i}(x) = \alpha_h(s^*)\}\big|\\
\geq &~ 2\epsilon\cdot B\log(\delta_1^{-1}).
\end{align}
Thus, by Chernoff bound, with probability at least $1-\cO(\delta_1)$, 
$
\mu_{k, i,h}(s^*)
\ge \epsilon\cdot \log(\delta_1^{-1}).$ Also note that $f^{k,i}_{h}$ is the first function that has confirmed on $s^*$ (i.e., no $\cD_{h,\alpha_h(s^*)}$ exists in $\cZ_h$ of line 8 at iteration $(k,i)$). By Line \ref{line:D''s} and Line \ref{line:update}, for later iterations, in $\cZ_h$, $\cD_{h,\alpha_h(s^*)}=\{x\in\cD''_{k,i,h}:f^{k,i}_h(x)=\alpha_h(s^*)\}$.
%Note that by definition of the label matching procedure, no previous functions have confirmed state $s$ ( i.e., no $\cD_s$ exists in $\cZ$ of line 8 at iteration $(k,i)$ ).

Next, for another $(k',i')>(k,i)$, we let the corresponding permutation be $\alpha_h'$ for $\tilde{f}_h^{k', i'}$. 
Since $\mu'(s') \ge  \mu_{k,i,h}(s') / (k'\cdot J)$, with probability at least $1-\delta_1$,
\begin{align}
\Pr_{s\sim \mu_{k,i,h},x\sim q(\cdot|s)}\big[\tilde{f}_{h}^{k',i'}(x)
\neq \alpha_h'\circ f_h(x)\big]
\le ~k'\cdot J \cdot g(\Theta(k'JB), \delta_1).
\end{align}
Notice that
\begin{align*}
\Pr_{s\sim \mu_{k,i,h},x\sim q(\cdot|s)}\big[\tilde{f}_{h}^{k',i'}(x)
\neq \alpha_h'\circ f_h(x)\big]
= &\sum_{s'\in \cS_h}\mu_{k,i,h}(s') \Pr_{x\sim q(\cdot|s')}\big[\tilde{f}_{h}^{k',i'}(x)\neq \alpha_h'\circ f_h(x)\big]\\
\ge &~\mu_{k,i,h}(s^*)\Pr_{x\sim q(\cdot|s^*)}\big[\tilde{f}_{h}^{k',i'}(x)
\neq \alpha_h'\circ f_h(x)\big]\\
\ge &~\epsilon\cdot \log(\delta^{-1}) \Pr_{x\sim q(\cdot|s^*)}\big[\tilde{f}_{h}^{k',i'}(x)
\neq \alpha_h'\circ f_h(x)\big].
\end{align*}
Thus, with probability at least $1-\delta_1$,
\begin{align}
\Pr_{x\sim q(\cdot|s^*)}&\big[\tilde{f}_{h}^{k',i'}(x)
\neq \alpha_h'\circ f_h(x)\big]
\le \frac{k'\cdot J \cdot g(\Theta(k'JB), \delta_1)}{\epsilon\cdot \log(\delta_1^{-1})}
\le \epsilon
\end{align}
with $B\geq B_0$ and $B_0$ as defined in Equation \eqref{eq:choiceB}.
Let $s':=\alpha_h'(s^*)$. 
Conditioning on $\ulo$ being correct on $\tilde{f}^{k', i'}_{[H+1]}$ and $f^{k,i}_{[H+1]}$, by Chernoff bound and Equation \eqref{eq:Ds*}, with probability at least $1-\cO(\delta_1)$, we have
\begin{align*}
\big|\{x\in\cD_{h,\alpha_h(s^*)}:  \tilde{f}_{h}^{k',i'}(x) = s'\}\big|
\ge &~ \big|\{x\in\cD_{h,\alpha_h(s^*)}:  f_h(x)=s^*, \tilde{f}_{h}^{k',i'}(x) = s'\}\big|\\
\ge &~(1-\epsilon \cdot\log(\delta_1^{-1}))\cdot\frac{2}{3}\cdot 
\big|\cD_{h,\alpha_h(s^*)}\big|
%\{x\in\cD''_{k,i,h}: f_{h}^{k, i}(x) = s\}\big|\\
> \frac{3}{5} \big|\cD_{h,\alpha_h(s^*)}\big|,
\end{align*}
where the fraction $\frac{2}{3}$ follows from Equation \eqref{eq:Ds*} and we use the fact that $\cD''_{k,i,h}$ are independent from the training dataset.
By our label fixing procedure, we find a permutation that swaps $s'$ with $s$ 
for $\tilde{f}_{h}^{k',i'}$ 
to obtain ${f}_{h}^{k',i'}$.
By the above analysis, with probability at least $1-\cO(\delta_1)$,
$
\Pr_{x\sim q(\cdot|s^*)}\big[f^{k',i'}_h(x)
\neq \alpha_h(s^*)\big]\le \epsilon
$
as desired.
Consequently, we let $\alpha_h^*(s^*) = \alpha_h(s^*)$,
which satisfies the requirement of the lemma.
\end{proof}
Next, by the definition of our procedure of updating the label standard dataset (Line \ref{line:update}, Algorithm \ref{alg:tsr}), we have the following corollary.
\begin{corollary}\label{coro:1}
Consider Algorithm~\ref{alg:tsr}.
Let $\cZ_{k, i,h}$ be the label standard dataset at episode $k$ before iteration $i$ for $\cS_h$.
Then,
with probability at least $1-\cO(H|\cS|\delta_1)$,
\begin{align}
\text{for all } k, i \text{ and }
\cD_{h,s}\in \cZ_{k, i,h},
|\{x\in \cD_{h,s}: \alpha_h^* \circ f_h(x) = s, s\in \cS_h\}|> 2/3 |\cD_{h,s}|.
\end{align}
\end{corollary}

At episode $k$ and iteration $i$ of the algorithm \tsr, let $\cE_{k,i}$ be the event that
for all $h\in[H+1], \cD_{h,s}\in \cZ_{k,i,h}$, $\Pr_{x\sim q(\cdot|s)}\big[f^{k,i}_h(x)
\neq \alpha_h^*\circ f_h(x)\big]\le \epsilon$.
We have the following corollary as a consequence of Lemma~\ref{lem:label} by taking the union bound over all states.
\begin{corollary} $\forall k,i:\quad \Pr\big[\cE_{k,i}\big] \ge 1-\cO(H|\cS|\delta_1).$
\end{corollary}

The next lemma shows that after $(H+1)|\cS|+1$ iterations of the \tsr~subroutine, the algorithm outputs a trajectory for the algorithm $\sA$ as if it knows the true mapping $f_{[H+1]}$.
\begin{lemma}
\label{lem:tsr}
Suppose in an episode $k$, we are running algorithm $\tsr$.
Then after $J=(H+1)|\cS|+1$ iterations, we have, for every $j\ge J$,
with probability at least $1-\cO(H|\cS|\delta_1)$,
\begin{align}\label{eq:good}
\text{for all }h\in[H+1], \Pr_{s\sim \mu_{k,j+1,h}, x\sim q(\cdot|s)}
\big[f_h^{k, j}(x)\neq \alpha_h^* \circ f_h(x)\big]\le \epsilon'
\end{align}
for some small enough $\epsilon$ and $50 H\cdot \epsilon\cdot |\cS|\cdot \log(\delta_1^{-1})<\epsilon'<1/2$,
provided $B\ge B_0$ as defined in Lemma \ref{lem:label}.
\end{lemma}
\begin{proof}[Proof of Lemma~\ref{lem:tsr}]
For $i<J$, there are two cases:
\begin{enumerate}
    \item there exists an $h\in[H+1]$ such that $\Pr_{s\sim \mu_{k,i+1, h}, x\sim q(\cdot|s)}\big[
f_h^{k, i}(x)\neq \alpha_h\circ f_h(x)
\big]
>\epsilon'/(2H)$;
\item for all $h\in[H+1]$, $\Pr_{s\sim \mu_{k,i+1, h}, x\sim q(\cdot|s)}\big[
f_h^{k, i}(x)\neq \alpha_h\circ f_h(x)
\big]
\leq\epsilon'/(2H)$,
\end{enumerate}
where $\alpha_h$ is some good permutations between $f^{k,i}_h$ and $f_h$. If case 1 happens, then there exists a state $s^*\in \cS_h$ such that
\begin{align}\label{eq:epsilon'}
\Pr_{x\sim q(\cdot|{s^*})}\big[f_h^{k,i}(x)\neq \alpha_h\circ f_h(x)\big]\cdot \mu_{k,i+1,h}(s^*)
>\frac{\epsilon'}{2H|\cS|}.
\end{align}
If $\cD_{h,\alpha_h(s^*)}\in \cZ_{k,i,h}$, where $\cZ_{k,i,h}$ is defined as in Corollary \ref{coro:1}, by Lemma~\ref{lem:label}, with probability at least $1-\cO(\delta_1)$,
\[
\Pr_{x\sim q(\cdot|s^*)}[
f_h^{k,i}(x) 
\neq \alpha_h^*\circ f_h(x)]
\le \epsilon
\]
and $\alpha_h^*(s^*) = \alpha_h(s^*)$.
Thus,
$
\mu_{k,i+1,h}(s^*)>\frac{\epsilon'}{2H|\cS|}/\epsilon >1
$,
a contradiction with  $\mu_{k,i+1,h}(s^*)\le 1$.
Therefore, there is no $\cD_{h,\alpha_h(s^*)}$ in $ \cZ_{k,i,h}$. 
Then, due to $\Pr_{x\sim q(\cdot|s^*)}\big[f_h^{k,i}(x)\neq \alpha_h\circ f_h(x)\big]\le 1$, by Equation \eqref{eq:epsilon'}, we have
\begin{align}\label{eq:mulower}
\mu_{k,i+1,h}(s^*)>\frac{\epsilon'}{2H|\cS|}.
\end{align}
Since $f^{k, i+1}_h$ is trained on $\unif(\{\mu_{k',i',h}\}_{(k',i')< (k,i+1)})$,
by Definition of $\ulo$,
with probability at least $1-\delta_1$,
\begin{align}\label{eq:fi+1}
\Pr_{s\sim \mu_{k,i+1,h},x\sim q(\cdot|s)}\big[
f^{k, i+1}_h(x)\neq \alpha_h'(s)
\big]
\le k\cdot J
\cdot g(\Theta(kJB), \delta_1)\le\epsilon^2,
\end{align}
with $B\geq B_0$ ($B_0$ is defined in Equation \eqref{eq:choiceB}) and $\alpha_h'$ is some good permutation between $f^{k,i+1}_h$ and $f_h$.
Thus, by Equation \eqref{eq:mulower} and the choice of $\epsilon$ and $\epsilon'$, we have 
\begin{align}
  \Pr_{x\sim q(\cdot|s^*)}\big[
f^{k, i+1}_h(x)\neq \alpha_h'(s^*)
\big]< \epsilon/25. 
\end{align}
Thus, 
\begin{align}
  \mu_{k,i+1,h}(s^*)\cdot\Pr_{x\sim q(\cdot|s^*)}\big[
f^{k, i+1}_h(x) = \alpha_h'(s^*)
\big]> \frac{\epsilon'}{2H|\cS|}\cdot (1-\epsilon/25)>24\epsilon\cdot\log(\delta_1^{-1}), \end{align}
where the last inequality is due to $\epsilon<\epsilon'<1$.
By Chernoff bound,
with probability at least 
$1-\cO(\delta_1)$,
\[
|\{x\in \cD''_{k, i+1,h}: f^{k, i+1}_h(x) = \alpha_h'(s^*)\}| \ge 3\epsilon \cdot B\log(\delta_1^{-1}).
\]
Therefore, if case 1 happens, one state $s$ will be confirmed in iteration $i+1$ and $\alpha_h^*(s^*) = \alpha_h'(s^*)$ is defined.

To analyze case 2, we first define sets $\{\cG_{k,i+1,h}\}_{h=1}^{H+1}$ with $\cG_{k,i+1,h}:=\{s\in\cS_h~|~\cD_{h,s}\in\cZ_{k,i+1,h}\}$, i.e., $\cG_{k,i+1,h}$ contains all confirmed states of level $h$ before iteration $i+1$ at episode $k$. If case 2 happens, we further divide the situation into two subcases:
\begin{enumerate}
    \item[a)] for all $h\in[H+1]$, for all $s\in \cG^c_{k,i+1,h}$, $\mu_{k,i+1,h}(s)\leq \epsilon'/(8H|\cS|)$;
    \item[b)] there exists an $h\in[H+1]$ and a state $s^*\in\cG^c_{k,i+1,h}$ such that $\mu_{k,i+1,h}(s^*)\geq \epsilon'/(8H|\cS|)$,
\end{enumerate}
First notice that for every $h\in[H+1]$ and $j>i$, since $f^{k,j}_h$ is trained on $\unif(\{\mu_{k',i',h}\}_{(k',i')\leq(k,j)})$, by Definition of $\ulo$ and our choice of $B$ in Equation \eqref{eq:choiceB}, with probability at least $1-\delta_1$, we have
\begin{align}\label{eq:correctj}
    &\Pr_{s\sim\mu_{k,i+1,h},x\sim q(\cdot|s)}[f^{k,j}_h\neq\alpha_h'(s)]\leq \epsilon^2,\\
    \Rightarrow &\sum_{s\in\cG_{k,i+1,h}} \mu_{k,i+1,h}(s)\Pr_{x\sim q(\cdot|s)}[f^{k,j}_h(x)\neq \alpha_h'(s)]+\sum_{s\notin\cG_{k,i+1,h}} \mu_{k,i+1,h}(s)\Pr_{x\sim q(\cdot|s)}[f^{k,j}_h(x)\neq \alpha_h'(s)]\leq\epsilon^2,
\end{align}
where $\alpha_h'$ is some good permutation between $f^{k,j}_h$ and $f_h$. 

If subcase a) happens, note that for $s\in\cG_{k,i+1,h}$, due to the FixLabel routine (Algorithm \ref{alg:label}), $\alpha_h'(s)=\alpha_h^*(s)$, for $f^{k,j}_h (j>i)$ we have
\begin{align}\label{eq:subcase1}
    &\sum_{s\in\cS_h}\mu_{k,i+1,h}(s)\Pr_{x\sim q(\cdot|s)}[f^{k,j}_h(x)\neq \alpha_h^*(s)]\\
    =&\sum_{s\in\cG_{k,i+1,h}} \mu_{k,i+1,h}(s)\Pr_{x\sim q(\cdot|s)}[f^{k,j}_h(x)\neq \alpha_h^*(s)]+\sum_{s\notin\cG_{k,i+1,h}} \mu_{k,i+1,h}(s)\Pr_{x\sim q(\cdot|s)}[f^{k,j}_h(x)\neq \alpha_h^*(s)]\\
    =&\sum_{s\in\cG_{k,i+1,h}} \mu_{k,i+1,h}(s)\Pr_{x\sim q(\cdot|s)}[f^{k,j}_h(x)\neq \alpha_h'(s)]+\sum_{s\notin\cG_{k,i+1,h}} \mu_{k,i+1,h}(s)\Pr_{x\sim q(\cdot|s)}[f^{k,j}_h(x)\neq \alpha_h^*(s)]\\
    \leq& ~\epsilon^2+\epsilon'/(8H)<\epsilon'/(4H).
\end{align}
Taking a union bound over all $f^{k,j}_{[H+1]}$, we have that for any $h\in[H+1]$, with probability at least $1-\cO(H\delta_1)$,
\begin{align}
    &\Pr_{s\sim\mu_{k,j+1,h},x\sim q(\cdot|s)}[f^{k,j}_h(x)=\alpha_h^*(s)]\geq \Pr_{s\sim\mu_{k,j+1,h},x\sim q(\cdot|s)}[f^{k,j}_h(x)=\alpha_h^*(s)=f^{k,i}_h(x)]\\
    &\geq \Pr_{\text{for all~} h'\in[h], s_{h'}\sim \mu_{k,j+1,h'},x_{h'}\sim q(\cdot|s_{h'})}[\text{ for all }h'\in[h], f^{k,j}_{h'}(x_{h'})=\alpha_{h'}^*(s)=f^{k,i}_{h'}(x_{h'})]\\
    &= \Pr_{\text{for all~} h'\in[h], s_{h'}\sim \mu_{k,i+1,h'},x_{h'}\sim q(\cdot|s_{h'})}[\text{ for all }h'\in[h], f^{k,j}_{h'}(x_{h'})=\alpha_{h'}^*(s)=f^{k,i}_{h'}(x_{h'})]\\
    &\geq 1-(\epsilon'/(2H)+\epsilon'/(4H))\cdot H \geq 1-\epsilon'.
\end{align}
Therefore, if case 2 and subcase a) happens, the desired result is obtained. 

If subcase b) happens, we consider the function $f^{k,i+1}_h$. By Equation \eqref{eq:correctj}, 
\begin{align}
    &\mu_{k,i+1,h}(s^*)\cdot\Pr_{x\sim q(\cdot|s^*)}[f^{k,i+1}_h(x)\neq \alpha_h'(s^*)]\leq\epsilon^2\\
    \Rightarrow &\Pr_{x\sim q(\cdot|s^*)}[f^{k,i+1}_h(x)\neq \alpha_h'(s^*)]\leq\epsilon^2/(\epsilon'/(8H|\cS|))\leq\epsilon,
\end{align}
where $\alpha_h'$ here is some good permutation between $f^{k,i+1}_h$ and $f_h$. Thus,
\begin{align}
  \mu_{k,i+1,h}(s^*)\cdot\Pr_{x\sim q(\cdot|s^*)}\big[
f^{k, i+1}_h(x) = \alpha_h'(s^*)
\big]> \frac{\epsilon'}{8H|\cS|}\cdot (1-\epsilon)>6\epsilon\cdot \log(\delta_1^{-1}).    
\end{align}
By Chernoff bound, with probability at least $1-\cO(\delta_1)$, $|\{x\in \cD''_{k, i+1,h}: f^{k, i+1}_h(x) = \alpha_h'(s^*)\}| \ge 3\epsilon \cdot B\log(\delta_1^{-1}).$
Therefore, the state $s^*$ will be confirmed in iteration $i+1$ and $\alpha_h^*(s^*)=\alpha_h'(s^*)$ is defined.

In conclusion, for each iteration, there are two scenarios, either the desired result in Lemma \ref{lem:tsr} holds already or a new state will be confirmed for the next iteration. Since there are in total $\sum_{h=1}^{H+1}|\cS_h|\leq (H+1)|\cS|$ states, after $J:=(H+1)|\cS|+1$ iterations, by Lemma \ref{lem:label}, with probability at least $1-\cO(H|\cS|\delta_1)$, for every $j\geq J$, for all $h\in[H+1]$ and all $s\in\cS_h$, we have $\Pr_{x\sim q(\cdot|s)}[f^{k,j}_h(x)\neq \alpha_h^*(s)]\leq \epsilon$. Therefore, it holds that for 
\begin{align}
    \Pr_{s\sim \mu_{k,j+1,h},x\sim q(\cdot|s)}(f^{k,j}_h(x)\neq \alpha_h^*(s))\leq \epsilon< \epsilon'.
\end{align}
\end{proof}

\begin{proposition}\label{prop:main}
	Suppose in Definition \ref{def:ulo}, $g^{-1}(\epsilon, \delta_1)=\poly(1/\epsilon, \log(\delta^{-1}_1))$ for any $\epsilon, \delta_1\in(0,1)$ and $\sA$ is $(\varepsilon, \delta_2)$-correct with sample complexity $\poly\left(|\cS|,|\cA|,H,1/\varepsilon,\log\left(\delta^{-1}_2\right)\right)$ for any $\varepsilon,\delta_2\in (0,1)$.
	Then for each iteration of the outer loop of Algorithm~\ref{alg:framework},
	the policy $\phi^n$ is an $\varepsilon/3$-optimal policy for the BMDP with probability at least $0.99$, using at most $\poly\left(|\cS|,|\cA|,H,1/\varepsilon\right)$ trajectories.
\end{proposition}
\begin{proof}[Proof of Proposition~\ref{prop:main}]
We first show that the trajectory obtained by running $\pi^k$ with  the learned decoding functions $f^{k,J}_{[H+1]}$ matches, with high probability, that from running $\pi^k$ with $\alpha_{[H+1]}^*\circ f_{[H+1]}$. 
Let $K=C(\varepsilon/4, \delta_2)$ be the total number of episodes played by $\sA$ to learn an $\varepsilon/4$-optimal policy with probability at least $1-\delta_2$. For each episode $k\in[K]$,  
let the trajectory of observations be $\{x^k_h\}_{h=1}^{H+1}$.
We define event $$\cE_{k}:=\{\forall h\in [H+1],
f^{k, J}_h(x^k_h) = \alpha_h^*(f_{h}(x^k_h))
\},$$
where $J=(H+1)|\cS|+1$.
Note that on $\cE_k$, the trajectory of running $\pi^{k}\circ \alpha^*_{[H+1]}\circ f_{[H+1]}$ equals running $\pi^{k}\circ f^{k, J}_{[H+1]}$.
We also let the event $\cF$ be that $\ulo$ succeeds on every iteration (satisfies Lemma~\ref{lem:tsr}).
Thus,
$$
\Pr[\cF] \ge 1- K\cdot J\cdot\delta_1=1-\poly(|\cS|, |\cA|, H, 1/\varepsilon, \log(\delta_1^{-1}))\cdot\delta_1.
$$
Furthermore, each $x_{k, h}$ is obtained by the distribution $\sum_{s}\mu_{k, J+1, h}(s)q(\cdot|s)$. On $\cF$, by Lemma \ref{lem:tsr}, we have 
\[
\Pr[f^{k, J}_h(x^k_h) = \alpha_h^*(f_{h}(x^k_h))] \le 
\epsilon'
\]
by the choice of $B$.
Therefore, $$\Pr[\cE_k|\cF]
\ge 1-(H+1)\epsilon'.$$
Overall, we have $$\Pr\Big[\cE_k,~ \forall k\in [K]\Big|\cF\Big]\ge 1-K(H+1)\epsilon'.$$
Thus, with probability at least $1-\delta_2-\poly(|\cS|, |\cA|, H, 1/\varepsilon, \log(\delta_1^{-1}))\cdot(\epsilon'+\delta_1)$, $\cA$ outputs a policy $\pi$, that is $\varepsilon/4$-optimal for the underlying MDP with state sets $\{\cS_h\}_{h=1}^{H+1}$ permutated by $\alpha_{[H+1]}^*$, which we denote as event $\cE'$. 
Conditioning on $\cE'$, since on a high probability event $\cE''$ with $\Pr[\cE'']\ge 1-(H+1)\epsilon'$, $\pi\circ f_{[H+1]}^{K, J}$ and $\pi\circ \alpha^*_{[H+1]} \circ f_{[H+1]}$ have the same trajectory, the value achieved by $\pi\circ f_{[H+1]}^{K, J}$ and $\pi\circ \alpha^*_{[H+1]}\circ f_{[H+1]}$
differ by at most $(H+1)^2\epsilon'$.
Thus, with probability at least $1-\delta_2-\poly(|\cS|, |\cA|, H, 1/\varepsilon, \log(\delta_1^{-1}))\cdot(\epsilon'+\delta_1)$, the output policy $\pi\circ f_{[H+1]}^{K, J}$
 is at least  $\varepsilon/4+\cO(H^2\epsilon')$ accurate,
 i.e.,
\begin{align}
 V_1^{*} - V_1^{\pi\circ f_{[H+1]}^{K, J}}
 \le& ~V_1^* - V_1^{\pi\circ \alpha^*_{[H+1]}\circ f_{[H+1]}}% - (H+1)^2\epsilon'] \\
 + ~\cO(H^2\epsilon') %+\poly(|\cS||\cA|H)\delta)\\
 \le~ \varepsilon/4~+~\cO(H^2\epsilon').
 %+ \poly(|\cS||\cA|H)\delta),
\end{align}
Setting $\epsilon'$, $\delta_1$, and $\delta_2$ properly, $V_1^{*} - V_1^{\pi\circ f_{[H+1]}^{K, J}}\leq\varepsilon/3$ with probability at least $0.99$. Since $1/\delta_1 = \poly(|\cS|, |\cA|, H, 1/\varepsilon)$ and $1/\epsilon=\poly(|\cS|, |\cA|, H, 1/\varepsilon,\log(\delta_1^{-1}))$, $B_0$ in Lemma \ref{lem:label} and Lemma \ref{lem:tsr} is $\poly(|\cS|, |\cA|, H, 1/\varepsilon)$. The desired result is obtained.
\end{proof}

Finally, based on Proposition \ref{prop:main}, we establish Theorem \ref{thm:main}.

\begin{proof}[Proof of Theorem \ref{thm:main}]
By Proposition \ref{prop:main} and taking $N=\lceil\log(2/\delta)/2\rceil$, with probability at least $1-\delta/2$, there exists a policy in $\{\phi^n\}_{n=1}^N$ that is $\varepsilon/3$-optimal for the BMDP. For each policy $\phi^n$, we take $L:=\lceil9H^2/(2\varepsilon^2)\log(2N/\delta)\rceil$ episodes to evaluate its value. Then by Hoeffding's inequality, with probability at least $1-\delta/(2N)$, $|\bar{V}_1^{\phi^n}-V_1^{\phi^n}|\leq\varepsilon/3.$
By taking the union bound and selecting the policy $\phi\in\argmax_{\phi\in\phi^{[N]}} \bar{V}_1^{\phi}$, with probability at least $1-\delta$, it is $\varepsilon$-optimal for the BMDP. In total, the number of needed trajectories is $N\cdot \sum_{k=1}^K\sum_{i=1}^J \big((k-1)J+i+1\big)B + N\cdot L = \cO(N\cdot K^2\cdot J^2\cdot B+N\cdot L)=\poly(|\cS|, |\cA|, H, 1/\varepsilon,\log(\delta^{-1})).$ We complete the proof.
\end{proof}

\section{Examples of Unsupervised Learning Oracle}\label{app:example}

\section{More about Experiments}\label{app:experiment}
\paragraph{Parameter Tuning} 
In the experiments, for OracleQ, we tune the learning rate and a confidence parameter; for QLearning, we tune the learning rate and the exploration parameter $\epsilon$; for PCID, we follow the code provided in \cite{du2019provably}, tune the number of clusters for $k$-means and the number of trajectories $n$ to collect in each outer iteration, and finally select the better result between linear function and neural network implementation.

\paragraph{Unsupervised Learning Algorithms}
In our method, we use OracleQ as the tabular RL algorithm to operate on the decoded state space and try three unsupervised learning approaches: 1. first conduct principle component analysis (PCA) on the observations and then use $k$-means (KMeans) to cluster; 2. first apply PCA, then use Density-Based Spatial Clustering of Applications with Noise (DBSCAN) for clustering, and finally use support vector machine to fit a classifier; 3. employ Gaussian Mixture Model (GMM) to fit the observation data then generate a label predictor. We call the python library \texttt{sklearn} for all these methods. During unsupervised learning, we do not separate observations by levels but add level information in decoded states. Besides the hyperparameters for OracleQ and the unsupervised learning oracle, we also tune the batch size $B$ adaptively in Algorithm \ref{alg:tsr}. In our tests, instead of resampling over all previous policies as Line \ref{line:tsr_generate1} Algorithm \ref{alg:tsr}, we use previous data. Specifically, we maintain a training dataset $\cD$ in memory and for iteration $i$, generate $B$ training trajectories following $\pi\circ f^{i-1}_{[H+1]}$ and merge them into $\cD$ to train $\ulo$. Also, we stop training decoding functions once they become stable, which takes $100$ training trajectories when $H=5$, $500\sim1000$ trajectories when $H=10$, and $1000\sim2500$ trajectories when $H=20$. Since this process stops very quickly, we also skip the label matching steps (Line \ref{line:tsr_label} to Line \ref{line:tsr_label_end} Algorithm \ref{alg:tsr}) and the final decoding function leads to a near-optimal performance as shown in the results.

\cut{\section{Experiment Details\label{app:experiment}}

}

\end{document}